%% file: main.tex
\documentclass[pdflatex,sn-mathphys-num]{sn-jnl}

\input{sec_packages}

\begin{document}

\title{The 4/$\delta$ Bound: Designing Predictable LLM-Verifier Systems for Formal Method Guarantee}


\abstract{The integration of Formal Verification tools with Large Language Models (LLMs) offers a path to scale software verification beyond manual workflows. However, current methods remain unreliable: without a solid theoretical footing, the refinement process acts as a black box that may oscillate, loop, or diverge. This work bridges this critical gap by developing an LLM-Verifier Convergence Theorem, providing the first formal framework with provable guarantees for termination in multi-stage verification pipelines. We model the interaction not as a generic loop, but as a sequential absorbing Markov Chain comprising four essential engineering stages: \texttt{CodeGen}, \texttt{Compilation}, \texttt{InvariantSynth}, and \texttt{SMTSolving}. We prove that for any non-zero stage success probability ($\delta > 0$), the system reaches the \texttt{Verified} state almost surely. Furthermore, because of the sequential nature of the pipeline, we derive a precise latency bound of $\mathbb{E}[n] \leq 4/\delta$. We stress-tested this prediction in an extensive empirical campaign comprising over 90,000 trials. The results match the theory with striking consistency: every run reached verification, and the empirical convergence factor clustered tightly around $C_f\approx 1.0$, confirming that the $4/\delta$ bound accurately mirrors system behavior rather than serving as a loose buffer. Based on this data, we identify three distinct operating zones -- marginal, practical, and high-performance -- and propose a dynamic calibration strategy to handle parameter drift in real-world environments. Together, these contributions replace heuristic guesswork with a rigorous architectural foundation, enabling predictable resource planning and performance budgeting for safety-critical software.}

\keywords{Formal Verification, Large Language Models, SMT Solvers, Bounded Model Checking, Automated Program Repair, Specification Synthesis, ESBMC}

\author*[1,2]{\fnm{Pierre} \sur{Dantas}}\email{pierre.dantas@gmail.com, \href{https://orcid.org/0000-0001-6390-9340}{0000-0001-6390-9340}}

\author[1,2]{\fnm{Lucas} \sur{Cordeiro}}\email{lucas.cordeiro@manchester.ac.uk, \href{https://orcid.org/0000-0002-6235-4272}{0000-0002-6235-4272}}
\equalcont{These authors contributed equally to this work.}

\author[1]{\fnm{Youcheng} \sur{Sun}}\email{techieyoucheng@gmail.com, \href{https://orcid.org/00000-0002-1893-6259}{0000-0002-1893-6259}}
\equalcont{These authors contributed equally to this work.}

\author[2]{\fnm{Waldir} \sur{Junior}}\email{waldirjr@ufam.edu.br, \href{https://orcid.org/0000-0003-3095-0042}{0000-0003-3095-0042}}
\equalcont{These authors contributed equally to this work.}

\affil*[1]{\orgdiv{Dept. of Computer Science}, \orgname{The University of Manchester}, \country{UK}}

\affil[2]{\orgdiv{Dept. of Electrical Engineering}, \orgname{Federal University of Amazonas (UFAM)}, \country{Brazil}}


\maketitle

\input{sec_acronyms}


\input{sec_intro}

\input{sec_related_works}

\input{sec_theory}

\input{sec_proposed_work_theorem}

\input{sec_methodology}

\input{sec_experiments}

\input{sec_discussion}

\input{sec_conclusion}

\input{sec_acknowledgment}

\newpage
\bibliography{references}

\end{document}

%% file: sec_packages.tex
\usepackage[T1]{fontenc}
\usepackage[utf8]{inputenc}

\usepackage[table]{xcolor}
\usepackage{textcomp}
\usepackage{graphicx}
\usepackage{wrapfig}

\usepackage{booktabs}
\usepackage{multirow}
\usepackage{tabularx}

\usepackage{comment}
\usepackage{enumitem}
    \setlist{nosep}
\usepackage{siunitx}
\usepackage{verbatim}

\usepackage{amsmath, amssymb, amsfonts}
\usepackage{amsthm}
\usepackage{mathtools}

\theoremstyle{thmstyleone}
\newtheorem{theorem}{Theorem}

\theoremstyle{thmstyletwo}

\theoremstyle{thmstylethree}

\newcommand{\equat}[1]{Eq.~\eqref{#1}}

\usepackage{algorithm}
\usepackage{algorithmicx}
\usepackage{algpseudocode}
\usepackage{listings}

\usepackage[title]{appendix}
\usepackage{pdflscape}

\usepackage[acronym]{glossaries}
\glsdisablehyper

\newcommand{\graytt}[1]{\colorbox{gray!15}{\texttt{#1}}}

%% file: sec_acronyms.tex

\newacronym{llm}{LLM}{large language model}
\newacronym{ml}{ML}{machine learning}
\newacronym{dnn}{DNN}{deep neural network}
\newacronym{api}{API}{application programming interface}
\newacronym{ai}{AI}{Artificial Intelligence}
\newacronym{dafny}{Dafny}{Deductive Algorithm Verification System}
\newacronym{fvel}{FVEL}{Formal Verification with Evolutionary Learning}
\newacronym{cve}{CVE}{Common Vulnerabilities and Exposures}
\newacronym{sva}{SVA}{SystemVerilog Assertions}
\newacronym{cicd}{CI/CD}{continuous integration and continuous deployment}
\newacronym{rtl}{RTL}{register-transfer level}
\newacronym{nlp}{NLP}{natural language processing}
\newacronym{smt}{SMT}{Satisfiability Modulo Theories}
\newacronym{gptq}{GPTQ}{generative pre-trained transformer quantization}
\newacronym{peft}{PEFT}{Parameter-Efficient Fine-Tuning}
\newacronym{lora}{LoRA}{low-rank adaptation}
\newacronym{fveler}{FVELER}{formal verification with evolutionary learning and enhanced reasoning}
\newacronym{formai}{FormAI}{Formal Verification with Artificial Intelligence}
\newacronym{ltl}{LTL}{Linear Temporal Logic}
\newacronym{ctl}{CTL}{Computation Tree Logic}
\newacronym{smc}{SMC}{Symbolic Model Checking}
\newacronym{bdd}{BDD}{Binary Decision Diagrams}
\newacronym{bmc}{BMC}{Bounded Model Checking}
\newacronym{esbmc}{ESBMC}{Efficient SMT-based Context-Bounded Model Checker}
\newacronym{goto}{GOTO}{Generalized Operational Three-Address Output}
\newacronym{ast}{AST}{Abstract Syntax Tree}
\newacronym{vc}{VC}{Verification Condition}
\newacronym{cegar}{CEGAR}{Counterexample-Guided Abstraction Refinement}
\newacronym{icl}{ICL}{In-Context Learning}
\newacronym{cpu}{CPU}{central processing unit}
\newacronym{gpu}{GPU}{graphics processing unit}
\newacronym{iot}{IoT}{Internet of Things}
\newacronym{rag}{RAG}{retrieval-augmented generation}
\newacronym{ssa}{SSA}{Static Single Assignment}
\newacronym{sat}{SAT}{Boolean Satisfiability Problem}
\newacronym{gpt4}{GPT-4}{Generative Pre-trained Transformer-4}
\newacronym{ease}{EASE}{Evaluation and Assessment in Software Engineering}
\newacronym{fpv}{FPV}{Formal Property Verification}
\newacronym{iqr}{IQR}{interquartile range}
\newacronym{vmc}{VMC}{Vectorized Monte Carlo}

%% file: sec_intro.tex
\newpage
\section{Introduction}\label{sec:intro}

Formal software verification is essential in areas such as aerospace~\citep{Cofer2018}, medical devices~\citep{Hatcliff2012}, and autonomous systems~\citep{Luckcuck2019}. It uses mathematical methods to check that systems work as intended. For example, techniques such as \gls{bmc}~\citep{Biere1999} analyze software to detect hidden problems, such as numerical errors and memory-safety issues. Regular testing often misses these kinds of issues. Tools such as \gls{esbmc}~\citep{Gadelha2018ESBMC, Cordeiro2012} help with this process~\citep{Clarke2018Handbook}. However, a significant specification bottleneck limits widespread adoption: the extensive manual effort required to formulate precise formal specifications from ambiguous requirements~\citep{Woodcock2009}.

We now have a few solid ways to handle this problem, thanks to the recent arrival of \glspl{llm}. These models are proving incredibly useful; for example, they can easily generate formal specifications~\citep{Liu2025PropertyGPT}, construct tricky loop invariants~\citep{Pirzada2024LLM}, and significantly speed up automated program repair~\citep{Tihanyi2025New}. Despite the promising results mentioned, a central problem remains: the process that improves the use of \glspl{llm} and verifiers does not always guarantee precise and reliable outcomes~\citep {Liu2025PropertyGPT, Pirzada2024LLM, Tihanyi2025New}. Although researchers have developed strong theoretical foundations for both \glspl{llm}~\citep{Vaswani2017Attention,wei2022emergent} and Formal Verification tools~\citep{Clarke2018Handbook,Biere1999,Gadelha2018ESBMC} individually, they have not yet mathematically analyzed how these components behave when used together in an iterative process~\citep{beckert2024towards,Huang2023}. This gap causes the system to behave unpredictably, making it unsafe to use in critical areas.

\textbf{The Fundamental Gap addressed in this work}: Unpredictable \gls{ai}-Verifier Interactions. We cannot predict how \gls{ai} and Verifiers will interact in multi-stage workflows. Mixing statistical \gls{ai} components with reliable, deterministic verification tools creates a core conflict that current methods have not solved mathematically. Even though specific tests look promising~\citep{Pirzada2024LLM,Lin2024FVEL}, we lack a reliable convergence theory for the repeated refinement cycles typical of modern verification pipelines. Most existing approaches treat LLM refinement as a generic ``black box'' loop, ignoring the sequential engineering dependencies required to build valid proofs. This missing piece creates three problems that prevent the deployment of these systems in real-world applications:

\begin{enumerate}[topsep=1ex, itemsep=1ex]
    \item \textbf{Unpredictable Termination}: Without mathematical guarantees that the sequential process will converge, the refinement can get stuck in intermediate stages (e.g., oscillating between syntax errors and logical failures) or wander off track. This unpredictability makes it unreliable for safety-critical scenarios like flight control, where a guaranteed end to the process is mandatory.

    \item \textbf{Unbounded Resource Consumption}: The absence of a defined step limit precludes the configuration of reliable timeouts. In a multi-stage pipeline, costs accumulate; without a theoretical ceiling, the process could run out of control, consuming all available memory or processing time until the application crashes.
    
    \item \textbf{Lack of Performance Predictability and Stability}: It is currently a guessing game. Engineers have no systematic way to connect an \gls{llm}'s capability to its total verification latency. Furthermore, the assumption that LLM performance is constant (stationary) is often violated in practice, leaving engineers without strategies to detect or mitigate performance drift during operation.
\end{enumerate}

These limitations can lead to serious problems in the real world, making systems unsafe and rendering resource planning impossible.

\textbf{Our Proposed Approach}: A Convergence Theorem for Sequential \gls{llm}-Verification. We tackle the problem of guarantees using a new Convergence Theorem based on a sequential absorbing Markov Chain. We model the verification process not as a generic loop but as a structured pipeline comprising four essential engineering stages: \graytt{CodeGen}, \graytt{Compilation}, \graytt{InvariantSynth}, and \graytt{SMTSolving}.

Figure~\ref{fig:convergence-framework} illustrates our integrated \gls{llm}-Verifier Convergence Framework. The diagram depicts the sequential flow from the initial \graytt{Unverified} state through four milestones -- \graytt{CodeGen} ($s_1$), \graytt{Compilation} ($s_2$), \graytt{InvariantSynth} ($s_3$), and \graytt{SMTSolving} ($s_4$) -- to the final \graytt{Verified} absorbing state ($s_5$). By modeling these dependencies as a sequential absorbing Markov Chain, where $\delta$ governs the transition to the next stage ($s_i \to s_{i+1}$), we derive rigorous theoretical guarantees ($\mathbb{E}[n] \le 4/\delta$). A feedback loop links real-world testing to the theoretical model, supporting the dynamic calibration strategy to handle parameter non-stationarity in deployment. By continuously monitoring the empirical success rate ($\hat{\delta}$) against the theoretical stability regions, this mechanism detects performance drift (e.g., when $\hat{\delta}$ drops into the marginal region) and triggers corrective actions -- such as context resets or temperature adjustments -- to restore the system to the optimal practical region.

\begin{figure}[htbp]
    \centering
    \caption{\gls{llm}-Verifier Convergence Framework. The model visualizes the sequential progression through four mandatory refinement stages ($s_1 \to s_4$), where $\delta$ represents the probability of advancing to the next stage and $1-\delta$ represents the probability of retrying the current stage. A feedback loop links real-world testing to the theoretical model, supporting the dynamic calibration strategy to handle parameter non-stationarity in deployment. By continuously monitoring the empirical success rate ($\hat{\delta}$) against the theoretical stability regions, this mechanism detects performance drift (e.g., when $\hat{\delta}$ drops into the marginal region) and triggers corrective actions -- such as context resets or temperature adjustments -- to restore the system to the optimal practical region}
    \includegraphics[width=\linewidth]{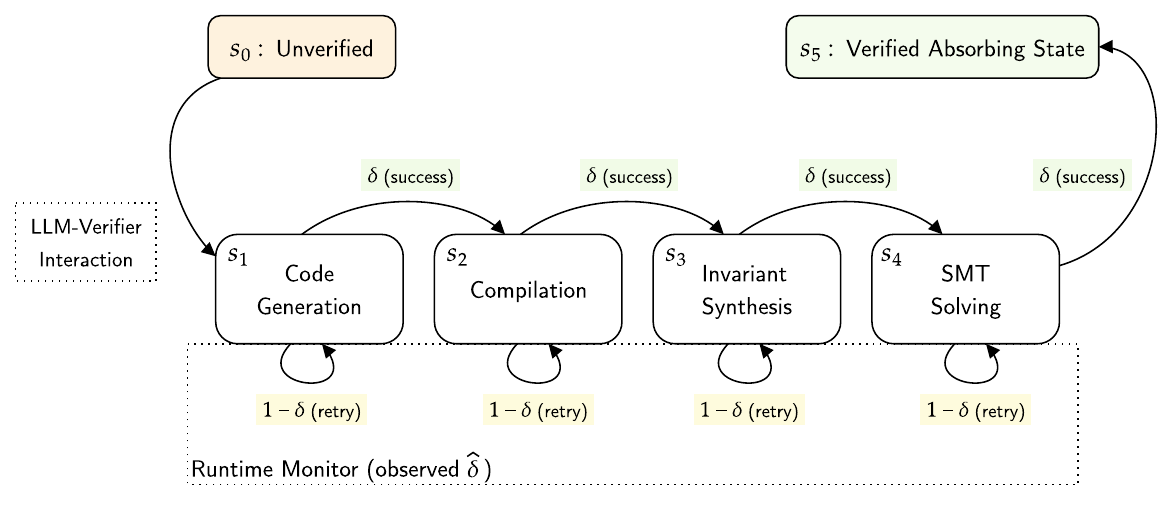}
    \label{fig:convergence-framework}
\end{figure}

In this model, $\delta$ quantifies the \gls{llm}'s single-attempt success rate for progressing from one stage to the next. A failure (probability $1-\delta$) necessitates a retry of the current stage, while success moves the artifact closer to the final \graytt{Verified} state. This structured view allows us to move beyond uncertain guesses~\citep{Zekri2024,Jianing2024} and prove that the system will definitely terminate.

Section~\ref{sec:main_result} contains the formal statement of our Theorem, which establishes that for any $\delta > 0$:

\begin{itemize}[topsep=1ex, itemsep=1ex]
    \item \textbf{Guaranteed Convergence}: If the program terminates and all reachable states are checked, the process is guaranteed to eventually traverse all four stages and reach the \graytt{Verified} state almost surely.
    
    \item \textbf{Bounded Latency}: The process is efficient. Because the system must pass through four distinct stages ($s_1 \to s_4$), we prove that the expected number of iterations is the sum of the expected time for each stage. This yields a precise bound of $\mathbb{E}[n] \le 4/\delta$, providing a worst-case ceiling for resource planning.
    
    \item \textbf{Exponential Tail Bounds}: The likelihood of the process needing more than $k$ steps drops off exponentially. This sharp decline enables reliable timeouts with high confidence.
\end{itemize}

To address the practical challenge of parameter non-stationarity (performance drift), our framework explicitly supports dynamic calibration. By continuously monitoring the empirical success rate against the theoretical $\delta$ regions, engineers can detect when an LLM drops into the marginal region and trigger context resets. Together, these results provide a comprehensive architectural foundation: our framework guarantees the process will finish, imposes a hard limit on computational costs ($4/\delta$), and provides a mechanism to maintain stability in dynamic environments.

To confirm our framework, we conducted an extensive study involving more than 90,000 simulations spanning the range 0.1 to 0.9. Our experimental results strongly support our theoretical predictions and show three distinct behavior patterns (see Table~\ref{tab:performance_regions}).

\begin{table}[htbp]
    \centering
    \caption{Performance regions based on $\delta$}
    \label{tab:performance_regions}
    \begin{tabular}{l c p{9cm}}
        \toprule
        \textbf{Region} & \textbf{Range} & \textbf{Description} \\
        \midrule
        \textbf{Marginal} & $\delta < 0.3$ & High variance; requires strict timeouts or dynamic calibration to prevent resource exhaustion. \\
        \textbf{Practical} & $0.3 \leq \delta \leq 0.6$ & Optimal for real-world verification tasks with balanced performance and predictable latency. \\
        \textbf{High-Performance} & $\delta > 0.6$ & Ideal for safety-critical systems requiring fast convergence and minimal variance. \\
        \bottomrule
    \end{tabular}
\end{table}

\textbf{Paper Organization}: This paper is structured as follows: Section~\ref{sec:related-work} reviews related work. Section~\ref{sec:theory} provides theoretical background. Section~\ref{sec:proposed_work} presents our convergence theory. Section~\ref{sec:methodology} details our validation methodology. Section~\ref{sec:experiments} discusses experimental results. Section~\ref{sec:discussion} interprets these findings. Finally, Section~\ref{sec:conclusion} summarizes the work and outlines future directions.

%% file: sec_related_works.tex

\section{Related Work} \label{sec:related-work}

The integration of \glspl{llm} into Formal Verification is a rapidly developing field built upon decades of foundational work in Formal Methods. This review examines current uses of \glspl{llm}, showing how practitioners integrate them, how these systems are structured, and how they perform. It shows how the field has developed over time, from traditional methods to modern systems that consider available resources.

\begin{enumerate}[topsep=1ex, itemsep=1ex]
    \item \textbf{Foundational and Classical Formal Methods (1999–2021)}: Early work centered on developing core verification technologies, such as \gls{bmc}~\citep{Biere1999}, and applying these techniques to safety-critical domains like medical devices~\citep{Hatcliff2012}, aerospace~\citep{Cofer2018}, and autonomous systems~\citep{Luckcuck2019}. However, the widespread adoption of these powerful methods stalled because the specification bottleneck forced practitioners to spend extensive manual effort translating ambiguous natural-language requirements into precise formal specifications~\citep{Woodcock2009}. The development of advanced tools, such as \gls{esbmc}~\citep{Gadelha2018ESBMC}, accelerated progress in the field; later updates expanded its capabilities, including support for C\texttt{++} templates~\citep{Monteiro2021}, and provided a strong software foundation. This foundation is essential for combining new statistical learning methods.

	\item \textbf{\gls{llm} Integration and Pattern Convergence (2023–2024)}: The rise of strong \glspl{llm} has started a new phase in research. This research is quickly focusing on three main ways to use these models together:

        \begin{enumerate}[topsep=1ex, itemsep=1ex]
        \item \textbf{Automated Specification and Invariant Generation}: This method uses \glspl{llm} to create formal rules and guidelines from everyday language or code, helping practitioners address a longstanding challenge in Formal Methods. Initially, the focus was on producing SystemVerilog Assertions to verify the hardware, which significantly improved coverage of the checks~\citep {Orenes2023Using}. Later work in 2024 also utilized \gls{llm}-generated invariants to improve the efficiency of \gls{bmc} itself~\citep{Pirzada2024LLM}.
        
        \item \textbf{Proof Synthesis}: This pattern involves employing \glspl{llm} to generate full or partial proofs for interactive theorem provers. Work in this field has progressed quickly. For example, systems like Baldur~\citep{First2023Baldur}  have achieved high accuracy rates and improved performance in environments such as Isabelle/HOL. Subsequent research in 2024 further enhanced proof synthesis in languages such as Dafny~\citep{Loughridge2024} by using iterative feedback and Chain-of-Thought prompting~\citep{Misu2024Towards, Wei2022}.
        
        \item \textbf{Dynamic, Feedback-Driven Approaches}: These systems deploy \glspl{llm} in a closed loop with verifiers, using error feedback to refine artifacts like specifications or proofs iteratively. Researchers established this approach using models like \gls{gpt4} alongside deductive verifiers~\citep{beckert2024towards}, and specialized frameworks such as \gls{fvel} later emerged to support fine-tuned theorem proving~\citep{Lin2024FVEL}.
    \end{enumerate}

    \item \textbf{Closed-Loop Systems (2025–Present)}: Recent studies are working on creating systems that can adjust and improve themselves automatically. A good example of this is the \gls{esbmc}-\gls{ai} framework, which helps fix code written in C using \gls{bmc}~\citep{Tihanyi2025New, Yiannis2024, Fakih2025}. Concurrently, research advanced property generation for C code~\citep{wang2025supporting}, \gls{rag} for smart contracts~\citep{Liu2025PropertyGPT}, automated function modeling~\citep{Deng2025LLM}, and Rust static analysis~\citep{Yao2023Leveraging}.
\end{enumerate}

The \gls{llm}-Verifier Convergence Theorem establishes the formal conditions under which a sequence of \gls{llm}-generated approximations converges to a verifier-correct solution. It moves beyond the empirical question of whether a specific prompt or fine-tuning technique works and provides a theoretical basis for why and how these methods succeed. With this foundation, we can recast the challenge of model compression not merely as preserving statistical fidelity, but as maintaining the logical convergence properties defined by the Theorem.

Table~\ref{tab:compressed_llm_advancements_chronological} summarizes the significant advancements in \gls{llm}-accelerated Formal Verification and model compression. 

\begin{table}[htbp]
\centering
\caption{Chronological advancement in \gls{llm}-accelerated Formal Verification}
\begin{tabular}{p{0.25\textwidth} p{0.33\textwidth} p{0.35\textwidth}}  
    \toprule  
    \textbf{Advancement} & \textbf{Description \& Key Features} & \textbf{Impact / Results}\\ 
    \midrule  
    \multicolumn{3}{c}{\textbf{Foundations and Early Integration (2023)}} \\
    \midrule
    Automated Proof Generation and Repair (Baldur)~\citep{First2023Baldur} & \glspl{llm} generate and repair whole proofs for Formal Verification (e.g., Isabelle/HOL) & State-of-the-art proof synthesis; 65.7\% of theorems proved automatically with Baldur+Thor\\\addlinespace

    Safety and Trustworthiness Verification~\citep{Huang2023} & Verification \& validation frameworks for \glspl{llm}, including runtime monitoring and regulatory compliance & Identifies gaps and proposes rigorous methods for \gls{llm} safety and trustworthiness\\\addlinespace
    
    Prompt Compression (LLMLingua)~\citep{Jiang2023}& Coarse-to-fine prompt compression for \glspl{llm}, maintaining semantic integrity at high compression ratios & Up to 20$\times$ prompt compression with little performance loss, accelerating inference\\
    \midrule
    
    \multicolumn{3}{c}{\textbf{Hybrid Systems and Early Optimization (2024)}} \\
    \midrule
    Model Compression Techniques~\citep{zhu2023survey, Wang2024meets, Dantas2024, Kim2025, Belhaouari2025, Li2025, Dantas2025} & Quantization, pruning, knowledge distillation, low-rank approximation, and hybrid methods for \glspl{llm} & Enable \gls{llm} deployment in resource-limited settings; up to 70\%+ compression with minimal loss\\\addlinespace
    
    \gls{llm}-Driven Invariant Generation for \gls{bmc}~\citep{Pirzada2024LLM} & \glspl{llm} generate loop invariants for \gls{bmc}, avoiding loop unrolling & Improves verification coverage for programs with unbounded loops using \gls{esbmc}\\\addlinespace

    Fine-Tuning and Inference Optimization~\citep{Wang2024meets, Kim2025} & Parameter-efficient fine-tuning (e.g., \gls{lora}), and inference optimization & Reduces resource overhead for adapting and deploying \glspl{llm}\\
    \midrule
    
    \multicolumn{3}{c}{\textbf{Advanced Closed-Loop and Specialized Compression (2025)}} \\
    \midrule
    
    Hybrid Verification (\glspl{llm} + Formal Methods)~\citep{Tihanyi2025vulnerability, Haoze2023, beckert2024towards} & Combines \glspl{llm}’ cognitive abilities with formal rigor for bug detection, invariant generation, and validation & Hybrid systems address scalability and soundness, outperforming standalone methods\\\addlinespace
    
    \gls{esbmc}-\gls{ai}: \gls{llm} + Bounded Model Checking~\citep{Tihanyi2025New, Yiannis2024} & Integrates \glspl{llm} with \gls{esbmc} for automated vulnerability detection and repair in C code & High-accuracy repair of buffer overflows, pointer errors\\\addlinespace

    Symbolic Compression for Interpretability~\citep{Lumen2025} & Formal symbolic compression framework for code generation and logical reasoning tasks & Achieves 78.3\% token compression and improves logical traceability by 62\%\\\addlinespace

    Efficient Self-Attention with Smart Pruning~\citep{Belhaouari2025} & Pruning and matrix folding in transformer layers for sustainable \glspl{llm} & 70\% overall model compression with stable or improved performance\\
    \bottomrule
\end{tabular}
\label{tab:compressed_llm_advancements_chronological}
\end{table}

%% file: sec_theory.tex
\section{Theoretical Background}\label{sec:theory}

This section establishes the theoretical foundation for the \gls{llm}-Verifier Convergence Theorem. We first detail the verification principles of the \gls{esbmc} framework, which provides the deterministic foundation for our model (Section~\ref{subsec:formal-verification}). We then characterize the integration of \glspl{llm} as generators of verification artifacts (Sections~\ref{subsec:llms_verification} and \ref{subsec:verifiers-llm-integration}). The main idea of our theory is to explain the refinement process using a model called an absorbing Markov Chain (see Section~\ref{subsec:markovian-properties}). The model builds on a concept called the error-reduction probability ($\delta$). It helps us show that the process will eventually solve, and it allows us to determine that we need about $4/\delta$ steps to do so. Our theory explains the refinement loop using an absorbing Markov Chain (see Section~\ref{subsec:markovian-properties}). This model builds on a concept called the error-reduction probability ($\delta$). It helps us show that the process is likely to improve over time and allows us to determine how many steps we need to take, namely $4/\delta$.

\subsection{Mathematical Foundation}
\label{subsec:math_foundation}

This subsection establishes the mathematical basis for simulating the convergence time, $\mathbb{E}[n]$, by leveraging the properties of the Geometric distribution applied to the Markov Chain's transient states.

\subsubsection{Modeling State Transitions with the Geometric Distribution}

Our methodology models the \gls{llm}-Verifier system as a five-state absorbing Markov Chain with four transient states ($s_1, s_2, s_3, s_4$) and one absorbing state ($s_5$, \graytt{Verified}). The simulation's efficiency relies on characterizing the time spent in each transient state.

\begin{itemize}[topsep=1ex, itemsep=1ex]
    \item \textbf{Residence Time ($M_j$)}: The duration of a single stay in a transient state $j$ is measured in discrete steps, or iterations. The number of iterations required to exit state $j$, denoted $M_j$, is modeled by a Geometric random variable. This distribution counts the number of independent Bernoulli trials (in this case, iterations) required to achieve the first ``success''. We define a success as a transition out of state $j$, and this transition occurs with probability $\delta$ in each iteration, as shown in \eqref{eq_probab_geom}.
    \begin{equation}\label{eq_probab_geom}
        P(M_j = k) = (1-\delta)^{k-1} \delta, \quad k \in \{1, 2, 3, \ldots\}
    \end{equation}
    
    \item \textbf{Expected Single-State Time}: The average number of interactions needed to exit a single state is the expected value of the Geometric distribution, as shown in \equat{eq_geom_distrib_expected}.
    \begin{equation}\label{eq_geom_distrib_expected}
        \mathbb{E}[M_j] = \frac{1}{\delta}
    \end{equation}

\end{itemize}

\subsubsection{The Vectorized Markov Chain and Total Convergence Time}

The total time to reach the absorbing state, $T$, is the sum of the independent residence times spent in the four transient states. This summation is the basis for the computational efficiency of the \gls{vmc} approach.

\begin{itemize}[topsep=1ex, itemsep=1ex]
    \item \textbf{Vectorization Basis}: Since the times $M_1, M_2, M_3, M_4$ are random and do not affect each other, we can find the total time for the $i$-th test, which we call $T_i$, by simply adding these times together, instead of simulating each step one by one., as shown in \equat{eq_vectorized_basis}.
    \begin{equation}\label{eq_vectorized_basis}
        T_i = M_{i,1} + M_{i,2} + M_{i,3} + M_{i,4}
    \end{equation}
    
    \item \textbf{Total Expected Convergence Time ($\mathbb{E}[n]$)}: According to the property of expectation, the average total interaction, denoted as $\mathbb{E}[n]$, is equal to the sum of the average times spent in each of the four stages., as shown in \equat{eq_e_n}.
    \begin{equation}\label{eq_e_n}
        \mathbb{E}[n] = \sum_{j=1}^{4} \mathbb{E}[M_j] = \frac{1}{\delta} + \frac{1}{\delta} + \frac{1}{\delta} + \frac{1}{\delta} = \frac{4}{\delta}
    \end{equation}
\end{itemize}

This equation, $\mathbb{E}[n] \le 4/\delta$, provides the core theoretical bound and is the target for empirical validation using the large-scale \gls{vmc} simulation.

\subsubsection{Markovian Properties of \gls{llm}-Verifier Interactions}
\label{subsec:markovian-properties}

The basis of our convergence theory is the description of how \gls{llm} verifiers interact using an absorbing Markov Chain~\citep{Ermon2014, Craig2002}. In this system, there are specific states known as absorbing states that, once reached, cannot be changed or left.

In our framework, the system reaches the \graytt{Verified} state ($s_5$) as its final absorbing goal. The path to this goal involves traversing a sequence of distinct transient states representing engineering milestones: \graytt{CodeGen} ($s_1$), \graytt{Compilation} ($s_2$), \graytt{InvariantSynth} ($s_3$), and \graytt{SMTSolving} ($s_4$). Recent studies show that \glspl{llm} naturally follow these step-by-step reasoning patterns~\citep{Zekri2024, Jianing2024}, which matches well with Formal Methods that describe ``verification'' as a process moving through defined refinement stages~\citep{Clarke1997, Clarke2018Handbook}.

Several lines of research support this Markovian characterization:

\begin{itemize}[topsep=1ex, itemsep=1ex]
	\item \textbf{\gls{llm} Reasoning as State Transitions}: Chain-of-Thought reasoning in \glspl{llm} naturally follows patterns of moving from one state to another. Each reasoning step depends only on the step before it, which matches the Markov property~\citep{Wei2022, zhou2022least, Zekri2024}.
   
    \item \textbf{Verification Processes as State Machines}: The modeling of verification workflows as state transition systems is well-established in Formal Methods literature~\citep{Clarke1997, Clarke2018Handbook, Woodcock2009}. Our work extends this foundation to incorporate the probabilistic nature of \gls{llm} interactions.
   
    \item \textbf{Probabilistic Program Analysis}: Previous research on probabilistic programming \citep{Staton2016} and probabilistic theorem proving \citep{Gogate2016} offers a strong theoretical foundation for studying systems with random components. These ideas apply directly to \gls{llm}-based refinement using Markov Chain models.
\end{itemize}

In our framework, the Markov property manifests through:
\begin{itemize}
    \item \textbf{State Dependence}: The probability of generating a correct verification artifact depends only on the current verification state and the \gls{llm}'s capability parameter $\delta$.
   
    \item \textbf{Memoryless Transitions}: Each refinement attempt is independent of the specific history of previous attempts, conditioned on the current state.
   
    \item \textbf{Stationary Transition Probabilities}: The error-reduction probability $\delta$ remains constant throughout the refinement process for a given \gls{llm}-verifier pair.
\end{itemize}

\paragraph{Absorbing Markov Chain Theory}
\label{subsec:markov-theory}

Two conditions formally define an absorbing Markov Chain: (i) there exists at least one absorbing state, and (ii) from every state it is possible to reach an absorbing state in finite steps. States that are not absorbing are called transient states~\citep{Ermon2014, Craig2002}.

This formulation extends the classical Kripke structure $M = (s, s_0, R, L)$ used in model checking by introducing probabilistic transitions (see Section~\ref{sec:bmc_transformation}). The two frameworks connect clearly: the set of states $S$ contains both temporary and final states, the initial states $s_0$ represent where verification begins, and the transition relation $R$ covers probabilistic moves between states.

The transition matrix $P$ for a chain with $t$ transient states and $r$ absorbing states has the canonical form, as shown in~\equat{eq:matrixP}, where $Q$ is a $t \times t$ matrix describing transitions between transient states, $R$ is a $t \times r$ matrix describing transitions from transient to absorbing states, and $I_r$ is the $r \times r$ identity matrix.
\begin{equation}\label{eq:matrixP}
    P = \begin{bmatrix}
    Q & R \\
    0 & I_r
    \end{bmatrix}
\end{equation}

The fundamental matrix $N = (I_t - Q)^{-1}$ provides the expected number of visits to transient states before absorption. The expected number of steps until absorption starting from transient state $i$ is given by the $i$-th entry of the vector, as shown in~\equat{eq:matrixtN1}, where $\mathbf{1}$ is a vector of ones.
\begin{equation}\label{eq:matrixtN1}
    \mathbf{t} = N\mathbf{1}
\end{equation}

For a specific chain structure analyzed in this paper, solving for the expected number of steps, $\mathbb{E}[n]$, yields the exact result $\mathbb{E}[n] = {(4 - 3\delta)}/{\delta}$. We use this expression to derive a conservative upper bound for resource estimation, as shown in~\equat{eq:upperbond}.
\begin{equation}\label{eq:upperbond}
    \mathbb{E}[n] = \frac{4 - 3\delta}{\delta} = \frac{4}{\delta} - 3 \leq \frac{4}{\delta}
\end{equation}

The probability of absorption in state $j$ when starting from transient state $i$ is given by the $(i,j)$-entry of the matrix $B = NR$. For large $k$, $P^k$ approximates these absorption probabilities, with $\lim_{k\to\infty} P^k$ revealing the long-term absorption behavior.

We base our convergence analysis of \gls{llm}-verifier refinement loops on this theoretical framework. In it, the \graytt{Verified} state is the final absorbing state, while the refinement states are temporary (transient) states. This framework extends traditional verification by adding probabilistic elements to capture the random behavior of \gls{llm}-based refinement.

\subsection{\Gls{esbmc} Formal Verification Principles}
\label{subsec:formal-verification}

Formal Verification uses math to prove that software works correctly according to given rules, checking all possible ways the software can run. Tools like \gls{esbmc} follow a straightforward process: they first transform the code into simpler forms, then apply \gls{bmc} using techniques such as loop unrolling and \gls{ssa}, and finally turn the properties into verification conditions. \gls{smt} solvers check these conditions to give definite answers or find errors.

These principles underpin our work on reliable \gls{llm}-verifier integration. Figure~\ref{fig:verification-architecture} illustrates the end-to-end verification pipeline of the \gls{esbmc} framework, showing the complete workflow from source code input to final verification results.
\begin{figure}[htbp]
    \centering
    \caption{The \gls{esbmc} verification framework architecture, integrating \gls{bmc} for finite-depth program analysis with \gls{smt} solving for efficient logical reasoning to enable rigorous software verification}
    \includegraphics[width=\linewidth]{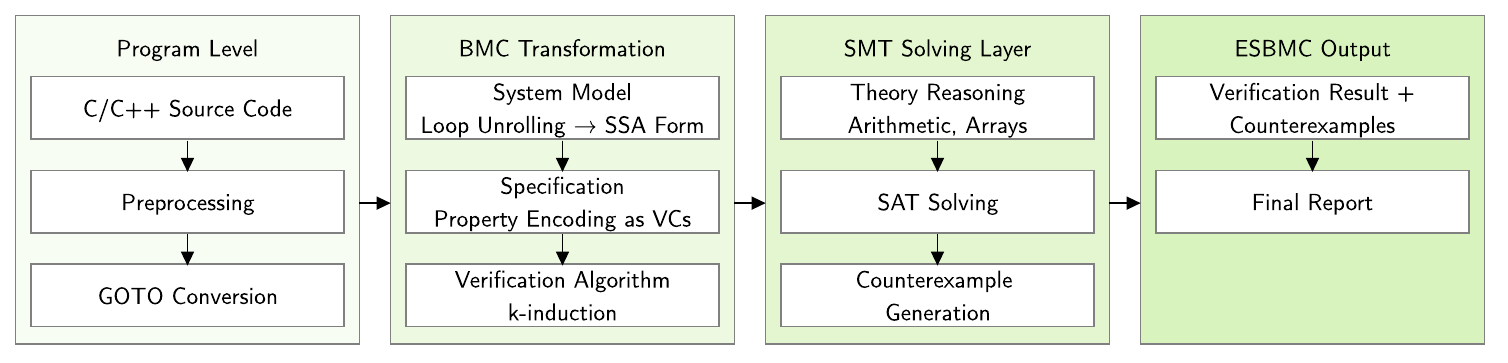}
    \label{fig:verification-architecture}
\end{figure}

\subsubsection{Program Representation Level}
\label{subsec:program-level}

The verification tool begins by transforming the source code into a form it can formally analyze. For C/C\texttt{++} programs, this transformation involves several critical preprocessing steps. Common approaches rely on intermediate representations that simplify complex language features while preserving semantic behavior~\citep{Kroening2016}.

The transformation pipeline starts with C/C\texttt{++} source code that includes complex language features such as preprocessor directives, pointers, dynamic memory allocation, and implementation-defined behavior. The preprocessing phase includes macro expansions and conditional compilation, producing a self-contained translation unit~\citep{Clarke2004}. Next, the process converts the code into the \gls{goto} intermediate representation, which simplifies complex control flow structures into basic blocks connected by explicit jumps~\citep{Bhm1966, Rozier2024}.

The \gls{goto} representation provides the base for the following \gls{bmc} steps. It removes many complex language details while preserving the essential meaning of the original program. This transformation separates language processing from Formal Verification, letting verification focus on the main computational logic.

\subsubsection{Logical Notations for \gls{bmc}}
\label{sec:logical-notations}

Formal Verification employs different logics to specify system properties. Propositional logic provides the basic connectives, while temporal logic introduces operators to reason about how truth values evolve. In these expressions, lowercase letters like \graytt{p} and \graytt{q} denote atomic propositions representing basic state conditions.

Temporal logic divides into two main types: \gls{ltl}, which reasons about linear sequences of states (single execution paths), and \gls{ctl}, which reasons about branching futures (all possible paths from a state). \Gls{bmc} primarily applies \gls{ltl} to formulate properties over finite execution paths.

Table~\ref{tab:logical-notations} summarizes the key operators, with a focus on those most relevant to \gls{bmc}.

\begin{table}[htbp]
    \centering
    \caption{Common Formal Verification notation: propositional logic and temporal logic}
    \label{tab:logical-notations}
    \begin{tabular}{>{\ttfamily}lllp{7.5cm}}
    \toprule
    \multicolumn{1}{l}{\textbf{Symbol}} & \textbf{Name} & \textbf{Logic Type} & \textbf{Meaning / Use Case} \\
    \midrule
    
    \multicolumn{4}{l}{\textbf{Propositional Logic}} \\
    \midrule
    \multicolumn{1}{l}{$\neg$\texttt{p} or $!$\texttt{p}} & Negation / NOT & Propositional & The statement \texttt{p} is false. \\
    \texttt{p} $\land$ \texttt{q} & Conjunction / AND & Propositional & Both \texttt{p} and \texttt{q} are true. \\
    \texttt{p} $\lor$ \texttt{q} & Disjunction / OR & Propositional & At least one of \texttt{p} or \texttt{q} is true. \\
    \texttt{p} $\rightarrow$ \texttt{q} & Implication & Propositional & If \texttt{p} is true, then \texttt{q} must be true. \\
    \texttt{p} $\leftrightarrow$ \texttt{q} & Equivalence / IFF & Propositional & \texttt{p} is true if and only if \texttt{q} is true. \\
    \midrule
    
    \multicolumn{4}{l}{\textbf{Linear Temporal Logic (LTL)}} \\
    \midrule
    $\mathbf{F}$\texttt{p} & Finally & LTL & Eventually, \texttt{p} will become true.\\
    $\mathbf{G}$\texttt{p} & Globally & LTL & \texttt{p} is always true, now and forever.\\
    $\mathbf{X}$\texttt{p} & Next & LTL & In the next state, \texttt{p} will be true. \\
    \texttt{p}$\mathbf{U}$\texttt{q} & Until & LTL & \texttt{p} remains true until \texttt{q} becomes true. \\
    \bottomrule
    \end{tabular}
\end{table}

\subsubsection{\Gls{bmc} Transformation}
\label{sec:bmc_transformation}

Model Checking is a formal, automated verification technique for determining whether a finite-state system model satisfies a given temporal logic specification. The core principle involves three fundamental components, which align with the initial stages of the \gls{bmc} transformation process depicted in the figure~\citep{Clarke1982, Clarke2008}:

\begin{enumerate}[topsep=1ex, itemsep=1ex]
    \item \textbf{System Model ($\textbf{M}$)}: The system under verification is formally modeled, typically as a Kripke structure or a transition system. This model $M$ represents all possible states and transitions that the system can undergo during its execution. Formally, a Kripke structure, illustrated in Figure~\ref{fig:kripkestructure}, is a tuple $M = (s, s_0, R, L)$, where:
    \begin{itemize}[topsep=1ex, itemsep=1ex]
        \item $s = \{s_1, s_2, s_3\}$ is a finite set of states.
        
        \item $s_0 \subseteq S = \{s_1\}$ is a set of initial states.
       
        \item $R \subseteq S \times S = \{(s_1, s_2), (s_2, s_1) (s_1, s_3), (s_3, s_3)\}$ is a transition relation that must be total.
      
        \item $L: S \to 2^{AP} = \{(s_1, \{p, q\}), (s_2, \{p,r\}), (s_3, \{q\})\}$ is a labeling function that assigns to each state a set of atomic propositions from a set $AP$ that are true in that state.
    \end{itemize}

    \begin{figure}[htbp]
        \centering
        \caption{An example of a Kripke structure}
        \includegraphics[width=0.35\linewidth]{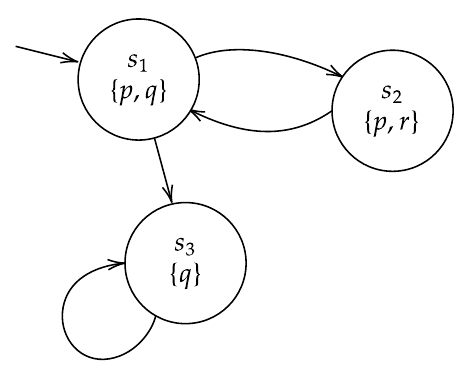}
        \label{fig:kripkestructure}
    \end{figure}
    
    \item \textbf{Specification ($\varphi$)}: The desired properties of the system, such as safety (\graytt{nothing bad happens}) and liveness (\graytt{something good eventually happens}), are expressed as formulas in a temporal logic, such as \gls{ltl}~\citep{Pnueli_1977} or \gls{ctl}~\citep{Emerson_1983}. For instance, the safety property \graytt{the system never enters a deadlock state} can be expressed in \gls{ltl} as $G(\neg \text{deadlock})$.

   \item \textbf{Verification Algorithm}: The algorithm checks whether the model $M$ satisfies the specification $\varphi$, denoted $M \models \varphi$, by verifying that all possible execution paths in $M$ satisfy the temporal logic formula $\varphi$. It then returns:

    \begin{itemize}
        \item \textbf{Yes}, if the property holds for all reachable states.
        
        \item \textbf{No, with a counterexample}, if the property is violated. This counterexample is a trace of states $\{ s_0, s_1, \dots, s_k\} $ leading from an initial state $s_0 \in s_0$ to a state $s_k$ that violates $\varphi$. This debugging aid is one of the most significant advantages of model checking.
    \end{itemize}
\end{enumerate}

The primary challenge in model checking is the state-space explosion problem, where the number of states in $M$ grows exponentially as the number of system components or variables increases. To solve this problem, researchers developed advanced techniques such as \gls{smc}~\citep{Vizel2015} based on \gls{bdd}~\citep{Bryant1986} and \gls{bmc}~\citep{Clarke2001}. The \gls{bmc} method forms the foundation of the \gls{esbmc}~\citep{Menezes2024} platform. This method uses a specific transformation pipeline:

\begin{enumerate}[topsep=1ex, itemsep=1ex]
    \item \textbf{Loop Unrolling and \gls{ssa} Form}: For \gls{bmc}, loops are unrolled to depth $k$ and the program is converted to \gls{ssa} form, assigning unique versions to variables to simplify state representation for logical encoding.
    
    \item \textbf{\Gls{bmc} and \gls{smt} Solving}: To address state-space explosion~\citep{Clarke2018Handbook, Clarke2012, Dhaussy2011}, \gls{bmc} verifies properties within finite bounds~\citep{Biere1999}, generating verification conditions solved via \gls{smt} solvers.

    \item \textbf{Property Encoding as \glspl{vc}}: The \gls{ssa} model and specification $\varphi$ are encoded into a \gls{vc} whose satisfiability implies a property violation within bound $k$~\citep{Mukherjee2016}. An unsatisfiable result indicates correctness within the bound.

    \item \textbf{From \gls{bmc} to Verification: \textit{k}-Induction}: $k$-induction extends \gls{bmc} to complete verification through base case checking (paths up to length $k$) and inductive step (if property holds for $k$ states, it holds for the next)~\citep{Donaldson2011}.    
    
    \begin{enumerate}[topsep=1ex, itemsep=1ex]
        \item \textbf{Base Case}: A standard \gls{bmc} check that verifies the property holds for all execution paths of length up to $k$.
       
        \item \textbf{Inductive Step}: An inductive proof that demonstrates if the property holds for any $k$ consecutive states, it must also hold for the next state.
    \end{enumerate}
\end{enumerate}

Studies have shown that this combined approach effectively balances bug-finding efficiency with verification completeness in tools like \gls{esbmc}~\citep{Gadelha2018ESBMC}.

\subsubsection{\acrfull{smt}}
\label{sec:smt_solving}

The application of \gls{bmc} to software verification requires reasoning about complex data types beyond propositional logic. \Gls{smt} solving extends \gls{sat} by combining boolean satisfiability with decision procedures for background theories~\citep{Barrett2018}. \Gls{smt} solvers serve as the computational engine that enables precise reasoning about program semantics -- including pointer arithmetic, memory operations, and concurrent execution -- by efficiently solving complex verification conditions generated from real-world software systems~\citep{Barreto2011}.

The architecture of modern \gls{smt} solvers comprises several key components working in concert:

\begin{enumerate}[topsep=1ex, itemsep=1ex]
    \item \textbf{\Gls{sat} Solving Core}: At its foundation, an \gls{smt} solver contains a modern \gls{sat} solver that handles the bBoolean structure of formulas. This component finds satisfying assignments to the propositional skeleton of the input formula by replacing theory atoms with boolean variables.

    \item \textbf{Theory Reasoning}: \Gls{smt} solvers integrate decision procedures for multiple background theories essential for software verification -- including bit-vector theory for fixed-width integers and bit-level operations, array theory for memory operations and data structures, linear arithmetic for integer and real constraints, and uninterpreted functions for abstract function symbols -- which work collaboratively to determine satisfiability of formulas involving complex data types~\citep{Brummayer2009, Ho2016, Tabuchi2003}.

    \item \textbf{Integration Architecture}: The \gls{sat} solver and theory solvers interact through a lazy integration scheme where the \gls{sat} solver proposes boolean assignments and theory solvers check their consistency, with theory conflicts driving the search process.

    \item \textbf{Counterexample Generation}: When the \gls{smt} solver finds a satisfiable verification condition, it produces a concrete counterexample -- an execution trace showing the sequence of states leading to the property violation, which is invaluable for debugging.
\end{enumerate}

The integration of \gls{bmc}, $k$-induction, and \gls{smt} solving forms a comprehensive verification framework where \gls{smt} solving provides the mathematical foundation for precise program reasoning.

\subsubsection{\Gls{esbmc} Output and Reporting}
\label{sec:esbmc_output}

The final stage of the \gls{esbmc} workflow produces clear, actionable output for users. This output directly corresponds to the results obtained from the \gls{smt} solving phase and consists of two primary components:

\begin{enumerate}[topsep=1ex, itemsep=1ex]
    \item \textbf{Verification Result + Counterexamples}: \gls{esbmc} provides definitive verification results and, when property violations are detected, generates detailed counterexamples comprising comprehensive execution traces that include the exact sequence of program statements, variable values at each step, thread interleavings for concurrent programs, and memory state information for pointer-related errors, which are invaluable for debugging and root cause analysis.

    \item \textbf{Final Report}: \gls{esbmc} compiles all verification results into a comprehensive final report providing a summary of all checked properties and their status (\graytt{success/failed}), total verification time and resource usage, statistical information, references to counterexamples for failed properties, and verification bounds and parameters, serving as complete documentation for quality assurance, certification, and regression testing.
\end{enumerate}

In summary, \gls{esbmc} implements the theoretical model-checking pipeline, demonstrating how engineers effectively combine \gls{bmc}, $k$-induction, and \gls{smt} solving to build a powerful tool for automated software verification. The framework's comprehensive output, including detailed counterexamples and final reports, provides users with actionable insights into program correctness and potential defects.

\subsection{\glspl{llm} for Verification}
\label{subsec:llms_verification}

\Glspl{llm} based on the Transformer architecture~\citep{Vaswani2017Attention} are trained on massive code and text corpora, enabling pattern recognition capabilities applicable to Formal Verification in \gls{esbmc}. It is crucial to emphasize that \glspl{llm} operate primarily through statistical pattern matching rather than logical reasoning. That program verification is fundamentally undecidable -- no algorithm can solve all verification problems. 

Despite their limitations, \glspl{llm} show two important abilities. First, they can recognize patterns that emerge, such as predicting code behavior and translating between natural language and formal representations~\citep{wei2022emergent}. Second, they can learn in context, adapting to verification tasks using carefully designed prompts without changing their parameters~\citep{brown2020language}.

In code analysis, \glspl{llm} use statistical patterns to generate candidate specifications, such as pre-conditions and loop invariants~\citep{pearce2021asleep, weiss2023thinking}. They also detect possible code problems by recognizing patterns~\citep{allamanis2021survey}. Their generative abilities produce verification artifacts, such as formal properties for \gls{esbmc}~\citep{roziere2023code} and boundary test cases~\citep{lemieux2023codamosa}. This approach builds a complete environment in which \glspl{llm} act as a statistical helper before specification generation and as an assistants after, explaining counterexamples~\citep{zhou2022least}.

For deployment in verification, \glspl{llm} require optimization via model compression and parameter-efficient fine-tuning. Compression includes quantization~\citep{Dettmers2022GPTQ, dettmers2023qlora}, pruning~\citep{Frantar2023SparseGPT}, and knowledge distillation~\citep{Hinton2015Distilling}. \Gls{peft} methods like \gls{lora} enable domain adaptation by injecting small trainable adapters while freezing most parameters~\citep{hu2021lora}, allowing efficient specialization for verification tasks. 

These pattern-matching capabilities, when integrated within our principled convergence framework, enable the development of predictable and reliable \gls{llm}-verifier systems. This integration acknowledges both the statistical nature of \glspl{llm} and the undecidability of verification, while providing theoretical guarantees for the iterative refinement process, making them suitable for deployment in safety-critical verification environments where both performance and predictable behavior are required.

Table~\ref{tab:compressed_llm_advancements} summarizes significant advancements in accelerating Formal Verification through compressed language models, highlighting the key techniques and their impacts on verification efficiency and capability.

\begin{table}[htbp]
\centering
\caption{Major advancements in applying \glspl{llm} to Formal Verification}
\begin{tabular}{p{0.20\textwidth} p{0.40\textwidth} p{0.34\textwidth}}
    \toprule
    \textbf{Advancement and Framework} & \textbf{Description \& Key Features} & \textbf{Impact / Results}\\
    \midrule
    
    Model Compression Techniques for \glspl{llm}~\citep{zhu2023survey, Kim2025, Belhaouari2025, Li2025, Wang2024meets, Dantas2024, Dantas2025} & Quantization, pruning, knowledge distillation for efficient pattern matching in \glspl{llm} & Enable \gls{llm} deployment in resource-limited settings; up to 70\%+ compression with minimal accuracy loss\\\addlinespace
    
    Prompt Compression (LLMLingua)~\citep{Jiang2023}& Coarse-to-fine prompt compression maintaining semantic integrity & Up to 20$\times$ prompt compression with little performance loss, accelerating inference\\\addlinespace
    
    Symbolic Compression for Interpretability~\citep{Lumen2025} & Formal symbolic compression framework for code generation tasks & Achieves 78.3\% token compression and improves traceability by 62\%\\\addlinespace
    
    Automated Proof Generation and Repair (Baldur)~\citep{First2023Baldur} & \glspl{llm} generate and repair proofs for Formal Verification systems & 65.7\% of theorems proved automatically with Baldur+Thor\\\addlinespace
    
    \gls{llm}-Driven Invariant Generation for \gls{bmc}~\citep{Pirzada2024LLM} & \glspl{llm} generate candidate loop invariants using pattern recognition & Improves verification coverage for programs with unbounded loops using \gls{esbmc}\\\addlinespace
    
    Hybrid Verification (\glspl{llm} + Formal Methods)~\citep{Tihanyi2025vulnerability, Haoze2023, beckert2024towards} & Combines \glspl{llm}' pattern recognition with formal rigor for bug detection & Hybrid systems address scalability while maintaining formal guarantees\\\addlinespace
    
    \gls{esbmc}-\gls{ai}: \gls{llm} + Bounded Model Checking~\citep{Tihanyi2025New, Yiannis2024} & Integrates \glspl{llm} with \gls{esbmc} for vulnerability detection and repair & High-accuracy repair of buffer overflows; suitable for \gls{cicd} integration\\\addlinespace
    
    Efficient Self-Attention with Smart Pruning~\citep{Belhaouari2025} & Pruning and matrix folding in transformer layers & 70\% overall model compression with stable performance\\\addlinespace
    
    Safety and Trustworthiness Verification~\citep{Huang2023} & V\&V frameworks for \glspl{llm}, including runtime monitoring & Identifies gaps and proposes methods for \gls{llm} safety assurance\\\addlinespace
    
    Fine-Tuning and Inference Optimization~\citep{Kim2025, Wang2024meets} & Parameter-efficient fine-tuning and inference optimization & Reduces resource overhead for adapting \glspl{llm} to verification tasks\\
    \bottomrule
\end{tabular}
\label{tab:compressed_llm_advancements}
\end{table}

\subsection{Formal Verifiers and \gls{llm} Integration}
\label{subsec:verifiers-llm-integration}

This subsection examines modern verifier characteristics enabling \gls{llm} integration and emerging collaboration patterns. Examples include industrial-strength model checkers like \gls{esbmc} for C/C\texttt{++} programs~\citep{Gadelha2018ESBMC, Monteiro2021, Menezes2024} and theorem provers in environments like \gls{fvel}~\citep{Lin2024FVEL}

Formal verifiers are crucial in the \gls{llm}-assisted verification ecosystem, acting as the ground truth for validating \gls{llm}-generated artifacts~\citep{beckert2024towards, Huang2023}. These automated tools use mathematical reasoning to prove or disprove software correctness with respect to formal specifications~\citep{Woodcock2009, Clarke2018Handbook}, employing techniques such as model checking~\citep{Clarke1997, Biere1999}, theorem proving, and abstract interpretation.

Unlike traditional testing, formal verifiers provide comprehensive coverage and mathematical certainty rather than probabilistic confidence~\citep{Clarke2001}. In \gls{llm} integration, they serve as reliable oracles to validate generated code~\citep{Misu2024Towards, Yiannis2024}, specifications~\citep{Liu2025PropertyGPT}, invariants~\citep{Pirzada2024LLM}, and repairs~\citep{Tihanyi2025New, Fakih2025}, creating a powerful approach between statistical \gls{ai} and formal reasoning~\citep{wang2025supporting, Haoze2023}.

Figure~\ref{fig_verifiers} illustrates the fundamental workflow of a formal verifier, depicting how it consumes a program and its specifications to deliver a definitive, mathematically sound result -- either a proof of correctness or a precise counterexample that guides debugging.
\begin{figure}[htpb]
    \centering
    \caption{The workflow of a formal verifier. This diagram illustrates the operational pipeline of a Formal Verification tool. The process begins with the software program and its formal specification being translated into a formal model. The verifier engine then performs a mathematical analysis of this model}
    \includegraphics[width=0.95\linewidth]{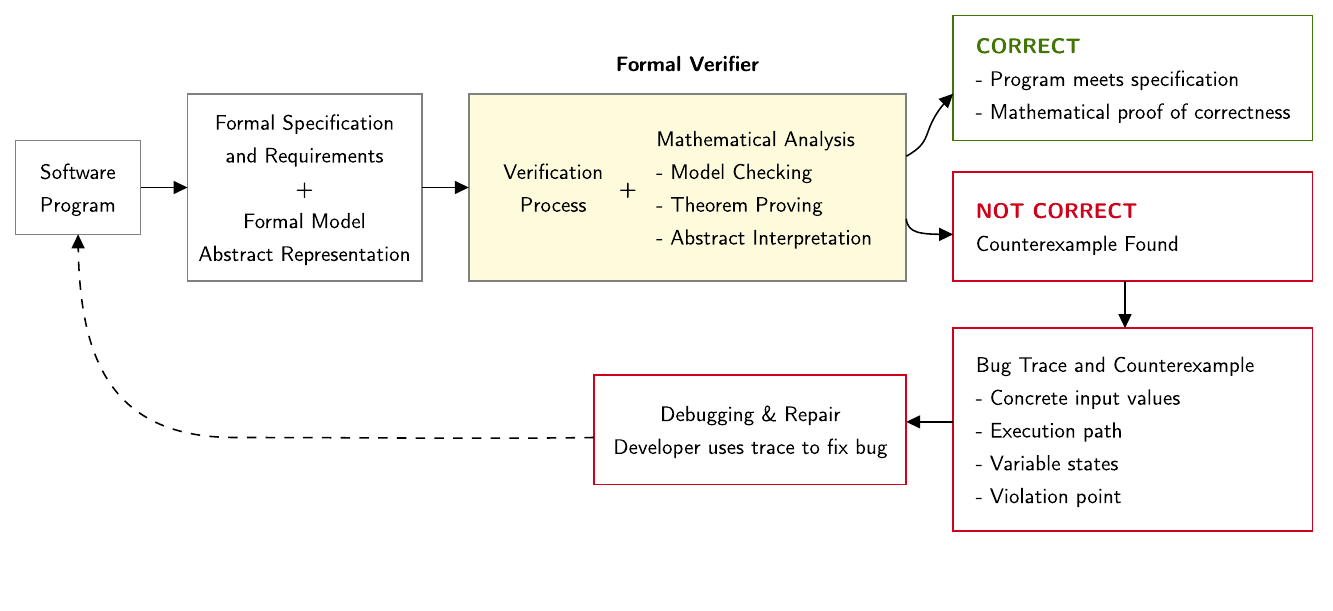}
    \label{fig_verifiers}
\end{figure}

Modern Formal Verification tools like \gls{esbmc} exhibit key characteristics enabling effective \gls{llm} integration~\citep{wang2025supporting, Haoze2023}, including deterministic output providing definitive answers with counterexamples~\citep{Clarke2001} for reliable \gls{llm} refinement~\citep{First2023Baldur, Yiannis2024}. These tools support incremental verification through bounded approaches~\citep{Cordeiro2012, Biere1999} and provide rich counterexamples for \gls{llm} learning~\citep{Gadelha2018ESBMC, Braue2022}, while scalable \gls{smt} integration enables practical verification of complex properties~\citep{Barrett2018, deMoura2008, Menezes2024}.

\gls{llm}-verifier integration uses different design patterns, such as sequential pipelines~\citep{Pirzada2024LLM, Liu2025PropertyGPT}, iterative refinement~\citep{First2023Baldur, Misu2024Towards}, hybrid verification~\citep{Tihanyi2025vulnerability, Haoze2023}, and verifier-as-oracle methods~\citep{Loughridge2024, Tihanyi2023FormAI}. These patterns help with important verification tasks - such as specification generation, invariant synthesis, program repair, and proof assistance - by clearly combining the strengths of statistical learning and Formal Methods.

\gls{esbmc}'s architecture is well-suited for these integration patterns~\citep{Gadelha2018ESBMC, Menezes2024}, featuring modular design for \gls{llm} assistance~\citep{Song2023, Monteiro2021}, comprehensive counterexamples for learning~\citep{Braue2022, Cordeiro2012}, and extensible verification strategies~\citep{Donaldson2011, Gadelha2020} that enable practical application to real-world programs~\citep{Monteiro2021, Barreto2011}.

Despite these architectural advances, \gls{llm}-verifier integration faces challenges including semantic gaps between statistical and formal reasoning~\citep{wei2022emergent}, effective feedback utilization~\citep{First2023Baldur}, scale mismatches~\citep{Clarke2012}, and uncertainty quantification in deterministic frameworks~\citep{Gogate2016}~\citep{beckert2024towards, Huang2023}. To solve these problems, researchers use some good methods. These include structured prompting~\citep{Wei2022, zhou2022least}, checking step-by-step (incremental verification)~\citep{Clarke2001, Donaldson2011}, and improved feedback loops~\citep{First2023Baldur, Gadelha2018ESBMC}, and specialized fine-tuning~\citep{hu2021lora, dettmers2023qlora}. 

This integration represents a fundamental shift toward collaborative systems leveraging statistical \gls{ai} and Formal Methods~\citep{Woodcock2009, Clarke2018Handbook}, providing mathematical guarantees for safety-critical deployment~\citep{Luckcuck2019, Cofer2018}.

%% file: sec_proposed_work_theorem.tex

\section{Proposed Work - \gls{llm}-Verifier Convergence Theorem}\label{sec:proposed_work}

The main goal of combining \glspl{llm} with Formal Verification tools like \gls{esbmc} is simple: Can we prove with math that repeated improvements will finally lead to a correct answer? To answer this, we build upon the Markov Chain theory (detailed in Section~\ref{subsec:markovian-properties}) to develop a new Theorem tailored for sequential \gls{llm}-verifier systems.

\subsection{Theoretical Foundation and Markov Chain Modeling}

Our convergence theory builds upon the rigorous framework of absorbing Markov Chains established in Section~\ref{subsec:markov-theory}. The classical theory guarantees that a Markov Chain with at least one absorbing state -- and where every state can eventually reach an absorbing state -- will be absorbed almost surely~\citep{Ermon2014,Craig2002}. 

This approach is well-supported theoretically and practically. Since transformer-based \glspl{llm} generate text token-by-token within limited context windows, they inherently display the properties of a Markov process. At the same time, empirical studies of systems like \gls{esbmc}-\gls{ai}~\citep{Tihanyi2025New} and Baldur~\citep{First2023Baldur} confirm that transition probabilities mainly depend on the current verification state.

This model uses a five-state Markov Chain with \graytt{Verified} as the absorbing state ($s_5$). Crucially, we replace generic refinement steps with four distinct transient states representing the essential milestones of a formal verification pipeline:
\begin{enumerate}[topsep=1ex, itemsep=1ex]
    \item $(s_1)$ \graytt{CodeGen}: The \gls{llm} attempts to generate the initial code or specification.
    \item $(s_2)$ \graytt{Compilation}: The artifact is validated for syntax and compilation errors.
    \item $(s_3)$ \graytt{InvariantSynth}: The system attempts to synthesize necessary loop invariants.
    \item $(s_4)$ \graytt{SMTSolving}: The final formal check where the solver attempts to prove correctness.
\end{enumerate}

The logic here mimics a sequential engineering workflow. When the system is in a specific stage $s_i$, it attempts to advance to the next milestone $s_{i+1}$. The transition succeeds with probability $\delta$. If it fails (with probability $1 - \delta$), the system remains in the current loop ($s_i$) to retry the specific task (e.g., regenerate code or refine invariants). The process reaches the absorbing \graytt{Verified} state only after successfully clearing the final stage $s_4$.

\textbf{The Error-Reduction Probability ($\delta$)}:
Central to our model is the error-reduction probability $\delta$, which quantifies the \gls{llm}'s capability to successfully complete a single pipeline stage in one attempt. Table~\ref{tab:llm_performance} shows typical $\delta$ values (0.360 - 0.880) for \gls{llm} verification tasks, validating its role in governing system convergence and providing the empirical foundation for our theory.
\begin{table}[htbp]
    \centering
    \caption{Summary of \gls{llm} performance on code-related tasks}
    \label{tab:llm_performance}
    \begin{tabular}{p{3.2cm}p{2.5cm}cp{6.5cm}}
        \toprule
        \textbf{Task} & \textbf{Model/System} & $\delta$ & \textbf{Details} \\
        \midrule
        Automated Code Repair & \gls{gpt4} & $0.543$~\citep{ease2024} & Single-attempt vulnerability repair on \gls{ease} benchmark (SecurityEval dataset)  \\
        \addlinespace
        Invariant Generation & \glspl{llm} (Inductive Loop Synthesis) & $0.780$~\citep{akhond2025} & First-attempt success rate in generating verified invariants \\
        \addlinespace
        Invariant Generation & CodeLlama-34b-Instruct & $0.360$~\citep{kamath2023} & First-attempt loop invariant synthesis (15 candidates, no Houdini pruning) \\
        \addlinespace
        Proof Synthesis & Baldur & $0.479$~\citep{First2023Baldur} & Single-attempt whole-proof generation on Isabelle/HOL \\
        \addlinespace
        Proof Synthesis & Baldur + Thor & $0.657$~\citep{First2023Baldur} & Single-attempt on Isabelle/HOL when combined with Thor \\
        \addlinespace
        Specification Synthesis & AssertLLM & $0.880$~\citep{Yan2025} & Syntax-valid and \gls{fpv}-verified assertion generation \\
        \addlinespace
        Specification Synthesis & AutoSpec & $0.790$~\citep{Wen2024} &  Generates usable specifications via \gls{llm} + static analysis + validation \\
        \bottomrule
    \end{tabular}
\end{table}

\subsection{The Main Result: \gls{llm}-Verifier Convergence Theorem}\label{sec:main_result}

Combining \glspl{llm} with verification tools, such as \gls{esbmc}, certainly helps automate invariant generation. However, a significant issue remains: most approaches lack theoretical backing. Without this foundation, the system becomes unstable. It may get stuck in infinite loops or oscillate indefinitely~\citep{beckert2024towards, Huang2023}. Our work fills this gap by introducing the \gls{llm}-Verifier Convergence Theorem. This Theorem offers the first formal proof of termination and convergence, which is essential for deployment in safety-critical systems~\citep{Cofer2018, Luckcuck2019}. Next, we present Theorem~\ref{thm:llm_verifier} and its proof.

\begin{theorem}[\gls{llm}-Verifier Convergence Theorem]\label{thm:llm_verifier}

We model the \gls{llm}-verifier process using a discrete-time Markov Chain, denoted as $X = \{X_n\}_{n\geq 0}$. The state space $S$ consists of a sequential engineering pipeline ($S = T \cup A$). First, $T = \{s_1, s_2, s_3, s_4\}$ represents the set of transient pipeline stages (\graytt{CodeGen}, \graytt{Compilation}, \graytt{InvariantSynth}, \graytt{SMTSolving}). Second, $A = \{s_5\}$ acts as the single absorbing state (\graytt{Verified}). Assuming a fixed success probability $\delta \in (0, 1]$ for passing any single stage, we construct the transition matrix $P = (P_{i,j})$. The individual entries \( P_{i,j} = \mathbb{P}(X_{n+1} = s_j \mid X_n = s_i) \) are arranged as follows:

\begin{enumerate}[label=(\roman*)]
    \item \textbf{Transient Pipeline States}: For the transient states $i \in \{1, 2, 3, 4\}$, the transition probabilities are:
    \[
    P_{i,j} = 
    \begin{cases} 
        \delta & \text{if } j = i+1 \quad \text{(success: advance to next stage),} \\
        1-\delta & \text{if } j = i \quad \text{(failure: retry current stage),} \\
        0 & \text{otherwise.}
    \end{cases}
    \]
    \item \textbf{Absorbing State}: State $s_5$ is absorbing, meaning the verification is complete. Hence, \( P_{5,5} = 1 \).
\end{enumerate}

\medskip
Let \( \tau = \inf\{ n \geq 0 : X_n \in A \} \) be the iteration count until verification. The following bounds hold:

\begin{enumerate}
    \item \textbf{Almost Sure Convergence}:
    \[
    \mathbb{P}(\tau < \infty \mid X_0 \in T) = 1
    \]

    \item \textbf{Expected Iteration Bound}:
    The mean time to convergence is given by:
    \[
    \mathbb{E}[\tau \mid X_0 = s_1] = \frac{4}{\delta}
    \]

    \item \textbf{Tail Bound}:
    Consider constants \( \alpha > 0 \) and \( \lambda_Q \in (0, 1) \), where \( \lambda_Q \) is the spectral radius of the transient submatrix \( Q \). Then, for all \( k \geq 0 \):
    \[
    \mathbb{P}(\tau > k \mid X_0 = s_1) \leq \alpha \lambda_Q^k.
    \]
    Substituting \( \lambda_Q = 1 - \delta \) results in the following exponential bound:
    \[
    \mathbb{P}(\tau > k \mid X_0 = s_1) \leq \alpha (1 - \delta)^k.
    \]
\end{enumerate}
\end{theorem}

\begin{proof}
We analyze the absorbing Markov Chain structure to establish the three guarantees.

\medskip

\noindent \textbf{Part 1: Almost Sure Convergence}: Does the system converge? To answer this, we check two requirements from Markov Chain theory. First, do we have an absorbing state? Yes, state $s_5$. Second, can every state reach it? Yes. From any point $s_i \in T$, the probability of reaching the end is $\delta^{5-i} > 0$. Meeting both criteria guarantees almost sure convergence. Thus,
\[
\mathbb{P}(\tau < \infty \mid X_0 \in T) = 1
\]

\medskip

\noindent \textbf{Part 2: Expected Iteration Bound}: We analyze the total iteration count $\tau$ by considering the sequential time spent in each pipeline stage. Let $M_j$ represent the number of steps the process remains in state $s_j$ before advancing. Because the probability of exiting $s_j$ to $s_{j+1}$ on any given step is $\delta$, the sojourn time $M_j$ follows a geometric distribution, $M_j \sim \text{Geom}(\delta)$, with an expected value of $\mathbb{E}[M_j] = 1/\delta$. Since the process is sequential, the total number of iterations is the sum of the time spent in all four stages, $\tau = \sum_{j=1}^4 M_j$. The linearity of expectation immediately gives us:
\[
\mathbb{E}[\tau \mid X_0 = s_1] = \sum_{j=1}^4 \mathbb{E}[M_j] = \frac{1}{\delta} + \frac{1}{\delta} + \frac{1}{\delta} + \frac{1}{\delta} = \frac{4}{\delta}
\]

\medskip

\noindent \textbf{Part 3: Tail Bound}: The transient submatrix $Q$ is upper triangular (not diagonal), as transitions only occur from $i$ to $i$ or $i+1$. Specifically, every diagonal entry is $P_{i,i} = 1-\delta$, and every super-diagonal entry is $P_{i,i+1} = \delta$. For an upper triangular matrix, the eigenvalues are simply the diagonal entries. Thus, the spectral radius is $\lambda_Q = \max_i |P_{i,i}| = 1-\delta$. Using standard Markov Chain theory, we find:
\[
\mathbb{P}(\tau > k \mid X_0 = s_1) \leq \alpha \|Q^k\| \leq \alpha (1 - \delta)^k
\]
Where $\alpha = \|(I-Q)^{-1}\|$. This equation establishes exponential decay.
\end{proof}

\subsection{A Practical Framework for Convergence}

The system's behavior is governed by the error-reduction probability $\delta$, a metric that quantifies the \gls{llm}'s ability to complete a specific verification sub-task. Unlike simpler models that assume a single ``generate-and-verify'' step, our framework acknowledges the engineering reality of a multi-stage pipeline. From any transient pipeline stage $s_i$ (e.g., \graytt{Compilation} or \graytt{InvariantSynth}), the process has a $\delta$ chance of successfully advancing to the next milestone $s_{i+1}$. Conversely, with a probability of $1-\delta$, the specific stage fails, necessitating a retry ($s_i \to s_i$) within the refinement loop before the system can proceed.

Our theory represents an important advancement in bridging \glspl{llm} and Formal Methods. We no longer have to depend on uncertain, black-box outcomes. Instead, we offer a straightforward and dependable framework backed by mathematical evidence:

\begin{itemize}[topsep=1ex, itemsep=1ex]
    \item \textbf{Resource Allocation}: Provides precise guidelines, specifically the bound $\mathbb{E}[n] = {4}/{\delta}$, to assist with setting timeout limits and planning computational budgets for the entire pipeline.

    \item \textbf{Capability Assessment}: Introduces a clear metric ($\delta$) to measure \gls{llm} performance on specific verification tasks, allowing for modular benchmarking of different stages (e.g., measuring $\delta$ for invariant generation separately from code repair).

    \item \textbf{Deployment Guidelines}: Establishes practical thresholds for use ($\delta > 0.3$ for standard systems and $\delta > 0.6$ for high-performance applications).
\end{itemize}

Mathematically, we confirm that for any $\delta > 0$, the system will reach the verified state almost surely. We can then use the derived formula $\mathbb{E}[n] = {4}/{\delta}$ to plan system limits, converting a theoretical measure into a usable engineering metric. Table~\ref{tab:iteration-bounds} shows how to use this logic to distribute resources effectively in various situations.

\begin{table}[htbp]
    \centering
    \caption{Expected total pipeline iterations for different \gls{llm} verification capabilities}
    \label{tab:iteration-bounds}
    \begin{tabular}{lcc}
        \toprule
        \textbf{\gls{llm} Capability Level} & \textbf{Success Rate $\delta$} & \textbf{Expected Total Iterations $\mathbb{E}[n]$} \\
        \midrule
        Moderate Capability & 0.3 & 13.3 \\
        Balanced Performance & 0.5 & 8.0 \\
        High Performance & 0.8 & 5.0 \\
        \bottomrule
    \end{tabular}
\end{table}

Furthermore, the exponential tail bound $\mathbb{P}(n > k) \leq \alpha(1 - \delta)^k$ ensures that the probability of excessively long refinement loops decays exponentially. This enables engineers to configure reliable timeouts for safety-critical systems, ensuring that the verification pipeline fails gracefully rather than hanging indefinitely.

%% file: sec_methodology.tex

\section{Methodology}
\label{sec:methodology}
Our methodology establishes a comprehensive framework for validating the convergence theory through systematic simulation and analysis. We designed this approach to test the theoretical bounds rigorously; it simultaneously characterizes the behavior of sequential \gls{llm}-verifier systems across their full operational range.

\subsection{Simulation Framework Design}
\label{subsec:simulation-framework}
We developed a specialized simulation framework that implements the sequential Markov Chain model described in Section~\ref{sec:proposed_work}. The simulator implements probabilistic transitions through the four distinct pipeline stages: \graytt{CodeGen} ($s_1$), \graytt{Compilation} ($s_2$), \graytt{InvariantSynth} ($s_3$), and \graytt{SMTSolving} ($s_4$), leading ultimately to \graytt{Verified} ($s_5$). In this model, $\delta$ governs the success probability of advancing from stage $s_i$ to $s_{i+1}$.

We examine the success probability parameter $\delta$ across its entire range. We categorize the resulting operational data into three key regions: marginal ($0.1 \leq \delta < 0.3$), practical ($0.3 \leq \delta \leq 0.6$), and high-performance ($0.6 < \delta \leq 0.9$). After generating the large-scale dataset, we proceed to the data collection phase. In this stage, we record both the total convergence time ($T$) and the sequence traces. The final stage involves the experimental validation framework, which uses a three-pronged analytical approach -- theoretical, sensitivity, and statistical -- to rigorously test the derived bound $\mathbb{E}[n] \le 4/\delta$ and to provide empirically validated engineering insights.

We summarize the complete experimental workflow in Figure~\ref{fig:methodology_overview}. The process begins with the simulation core, which translates the sequential pipeline model (Section~\ref{sec:proposed_work}) into a highly efficient \gls{vmc} simulation (Section~\ref{subsec:simulation-execution}).

\begin{figure}[htbp]
    \centering
    \caption{Overview of the \gls{llm}-Verifier Convergence validation methodology. The process links the theoretical sequential Markov Chain model to the experimental framework. The simulation core uses a \gls{vmc} approach with the Geometric distribution ($\text{Geom}(\delta)$) to efficiently generate $N=10,000$ trials across the three defined operational regions. After data collection (generating convergence time $T$ and computational metrics), the results are channeled into the experimental validation framework, which employs three complementary analytical lenses to confirm theoretical alignment, assess framework robustness, and ensure statistical reliability (see Table~\ref{tab:validation-framework})}
     \includegraphics[width=\textwidth]{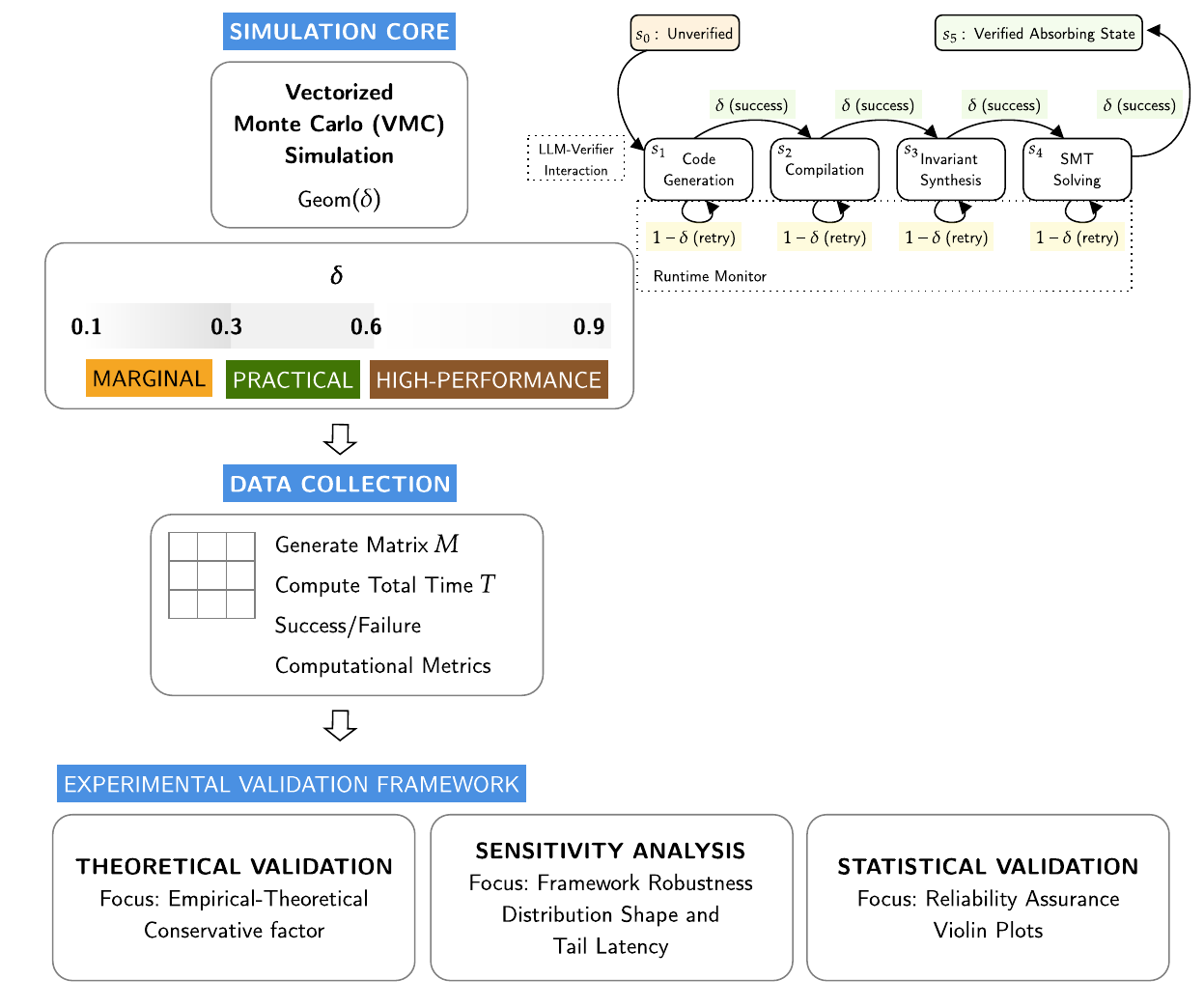}
    \label{fig:methodology_overview}
\end{figure}

We use a structured method to carefully check the accuracy of our theory about how well different verification methods for \glspl{llm} work. We test this across all available verification options to make sure our findings are reliable.
\begin{itemize}[topsep=1ex, itemsep=1ex]
    \item \textbf{Sample Size Justification}: The selection of 10,000 trials for each $\delta$ value is a robust methodological decision, ensuring high statistical power~\citep{Cohen2013,Robert2010} and narrow 99\% binomial confidence intervals (width $\le \pm 0.01$)~\citep{Agresti1998,Brown2001}. This large sample size is also critical for accurately estimating tail probabilities~\citep{Boucheron2013, Papaspiliopoulos2020} and for precisely computing higher-order empirical moments, such as variance and skewness~\citep{Wasserstein1997}.
    
    \item \textbf{Operational Regions Definition}: The parameter space $\delta \in [0.1, 0.9]$ is divided into three operational regions based on probabilistic analysis and empirical modeling.
 
    \begin{enumerate}[topsep=1ex, itemsep=1ex]
    
        \item The \textbf{marginal region} $(0.1 \leq \delta < 0.3)$ addresses small success probabilities where convergence is theoretically ensured by exponential tail bounds, $\mathbb{P}(n > k) \le e^{-c\delta k}$~\citep{Boucheron2013, Papaspiliopoulos2020}, but tests practical computational limits. 
        
        \item The \textbf{practical region} where \(\delta\) falls between 0.3 and 0.6 matches typical success rates seen in Bernoulli-type outcomes~\citep{Durrett2019}. This range also demonstrates the capabilities of \glspl{llm} in verification tasks~\citep {jiang2023codellama, First2023Baldur}. Therefore, it is the most important range when using these models in real-world situations.
        
        \item Finally, the \textbf{high-performance region} $(0.6 <\delta \leq 0.9)$ represents near-expert behavior, corresponding to the high-probability convergence regime in Monte Carlo and reliability theory~\citep{Robert2010, Wasserstein1997}.
    \end{enumerate}
 
    \item \textbf{Experimental Scale}: Is structured around 90,000 Monte Carlo trials (10,000 trials across nine $\delta$ values from 0.1 to 0.9) to ensure statistically rigorous coverage for convergence analysis~\citep{Wasserstein1997, Robert2010}. This method of systematic sampling helps us explore the area of success probabilities in a consistent way~\citep{Durrett2019}. It allows us to understand how changes in $\delta$ relate to the rate of convergence and makes it easier to compare different situations~\citep{Montenegro2006}. The scale we use helps us test the rule that says the average value of $n$ is less than or equal to 4 divided by $\delta$. It also helps us see sudden changes, which we call phase transitions~\citep{Durrett2019}. Overall, this creates a unique dataset that will be useful for future studies.
\end{itemize}

\subsection{Simulation Execution}
The validation follows a systematic protocol:

\begin{enumerate}[topsep=1ex, itemsep=1ex]
    \item \textbf{\gls{vmc} Simulation}: Our framework works differently than traditional simulations that run step-by-step. Instead, it uses a method called \gls{vmc} to speed up and improve efficiency. With \gls{vmc}, we can calculate the time for all 10,000 trials simultaneously, rather than one at a time. This approach treats groups of data as a single unit, which speeds up calculations.
    
    \item \textbf{Geometric Distribution}: The Geometric distribution, $\text{Geom}(\delta)$, models the residence time spent in each of the four pipeline stages ($s_1$ through $s_4$). Its expected value, $1/\delta$, directly establishes the theoretical basis for the primary bound we are validating. Since the process is sequential, the total expected time is the sum of the means: $\mathbb{E}[n] = \sum_{j=1}^{4} \mathbb{E}[M_{j}] = 4/\delta$.
    
    \item \textbf{Parameter Sweep}: For each $\delta$ value in $\{0.1, 0.2, \ldots, 0.9\}$, execute 10,000 independent trials.
    
    \begin{itemize}[topsep=1ex, itemsep=1ex]
        \item \textbf{Number of iterations until absorption}: This acts as the main measurement used to check the important theory that states $\mathbb{E}[n] \leq 4/\delta$. It directly shows how quickly things come together or reach a stable point.
        
        \item \textbf{Sequence of state transitions}: Enables analysis of the convergence path, identifying bottlenecks in specific pipeline stages (e.g., verifying if Invariant Synthesis $s_3$ consumes disproportionate resources).
        
        \item \textbf{Convergence} \graytt{success}/\graytt{failure}: Provides binary outcomes to empirically verify the Theorem's guarantee of almost sure convergence across all trials.
        
        \item \textbf{Computational metrics (execution time, memory usage)}: Look at how doable the plan is and how many resources it will need. Make sure to connect the ideas from theory with the actual challenges we face when putting them into practice.
    \end{itemize}

    \item \textbf{Statistical Configuration}: The framework features seeded random number generation for reproducible outcomes, independent trial execution to prevent cross-contamination, and comprehensive logging of all experimental parameters and outcomes.
\end{enumerate}

\subsubsection{Experimental Validation Framework}
\label{subsubsec:experimental-validation}
Our experimental study combines multiple methods, numerical measures, and key visuals to evaluate convergence theory from multiple angles carefully. The design of our experiment includes three methods to confirm our findings, each with its own measures and visuals to help interpret the results, as shown in Table~\ref{tab:validation-framework}.

\begin{table}[htbp]
    \centering
    \caption{Experimental validation framework: approaches, metrics, and visualizations}
    \label{tab:validation-framework}
    \begin{tabular}{p{1.6cm}p{4cm}p{4.5cm}p{3.5cm}}
        \toprule
        \textbf{Validation Approach} & \textbf{Methodological Focus} & \textbf{Metrics} & \textbf{Visualization} \\
        \midrule
        \textbf{Theoretical Validation} & 
        Empirical-theoretical alignment through statistical tests of distribution structure, direct bound comparison ($\mathbb{E}[n] \le 4/\delta$), and success rate verification across all $\delta$ settings. & 
        \textbf{Conservative Factor ($C_f$)}: Ratio of theoretical bound to empirical mean ($C_f = {(4/\delta)}/{\mu}$). \textbf{Efficiency ($\eta$)}: Successful state transitions per iteration ($\eta = 4/\mu$). \textbf{Success Rate}: Proportion of trials satisfying $n \leq 1000$. & 
        \textbf{Theoretical vs. Empirical Bounds} plot comparing $\mathbb{E}[n] \leq 4/\delta$ against empirical means (Figure~\ref{graph1_alignment}) \\
        \addlinespace
        \textbf{Sensitivity Analysis} & 
        Framework robustness assessment via systematic $\delta$ variation, boundary-condition testing ($\delta = 0.1, 0.9$), and distribution characteristic analysis. & 
        \textbf{Distribution Shape}: Variance ($\sigma^2$), Skewness ($\gamma_1$), and Kurtosis ($\kappa$). \textbf{Spread}: \gls{iqr} ($P_{75}-P_{25}$). \textbf{Tail Latency ($P_{99}$)}: The 99th percentile of iterations to absorption. & 
        \textbf{Operational Regions Map} demarcating marginal ($\delta < 0.3$), practical ($0.3 \leq \delta \leq 0.6$), and high-performance ($\delta > 0.6$) regions (Figure~\ref{graph2_regions})\\
        \addlinespace
        \textbf{Statistical Validation} & 
        Reliability assurance through confidence interval calculation, hypothesis testing against theoretical predictions, and statistical power analysis. & 
        \textbf{Precision ($W_{99}$)}: Width of the 99\% Confidence Interval ($2.576 \cdot \sigma / \sqrt{N}$). \textbf{Throughput}: Simulation speed (trials/sec). \textbf{Resource Cost}: Peak memory allocation (MB). & 
        \textbf{Tail Probability Analysis} with bounds $e^{-c\delta k}$ (Figure~\ref{graph3_tail_prob}). \textbf{Iteration Distribution Characterization} violin plots (Figure~\ref{graph4_violin})\\
        \bottomrule
    \end{tabular}
\end{table}

This combined approach turns theoretical ideas into real-world engineering knowledge. It helps ensure the safe use of language model verifiers by providing different ways to analyze and understand them.

\subsubsection{Simulation Execution Engine}
\label{subsec:simulation-execution}
Unlike iterative state-machine simulations, our framework utilizes a vectorized Monte Carlo approach to maximize computational throughput and statistical power.
\begin{enumerate}[topsep=1ex, itemsep=1ex]
    \item \textbf{Vectorized State Generation}: For each configuration $\delta \in \{0.1, 0.2, \ldots, 0.9\}$, the system instantiates a matrix $\mathbf{M} \in \mathbb{Z}^{N \times 4}$ where $N=10,000$ (the trial count). Each element $\mathbf{M}_{i,j}$ represents the residence time in stage $j$ for trial $i$, sampled independently from a Geometric distribution, as defined in \equat{eq:geometric_dist}.
    \begin{equation}\label{eq:geometric_dist}
        \mathbf{M}_{i,j} \sim \text{Geom}(\delta) \quad \forall i \in [1, N], j \in [1, 4]
    \end{equation}

    \item \textbf{Iteration Aggregation}: The total convergence time $T_i$ for the $i$-th trial is computed via row-wise summation, $T_i = \sum_{j=1}^{4} \mathbf{M}_{i,j}$, eliminating loop overhead and ensuring memory contiguity.

    \item \textbf{Resource Monitoring}: To accurately capture computational metrics without observer effect, the simulation encapsulates the generation phase in a \graytt{tracemalloc} context, which records peak memory allocation (MB) and process time (\gls{cpu} seconds) for each $\delta$ batch.
\end{enumerate}

\subsubsection{Data Collection and Filtering}
\label{subsec:data-collection}
Our processing pipeline first extracts trial-level granularity from the raw data of each simulation batch before performing aggregation:
\begin{itemize}[topsep=1ex, itemsep=1ex]
    \item \textbf{Convergence Status}: A trial is called a \graytt{Success} if it takes 1,000 units of time or less. This limit helps prevent the trial from running forever when the value is very low (less than 0.1). It is also much higher than what we usually expect, which is about 40 units of time when the value is 0.1.

    \item \textbf{Sequence Trace}: Saving the complete list $[\mathbf{M}_{i,1}, \mathbf{M}_{i,2}, \mathbf{M}_{i,3}, \mathbf{M}_{i,4}]$ for each trial enables us to identify specific patterns where progress slows. For example, we can identify trials that clear the initial CodeGen and Compilation stages quickly but stall during the final SMT Solving stage.
\end{itemize}

\subsubsection{Metric Computation Methodology}
\label{subsec:metric-computation}
The quantitative metrics presented in Table~\ref{tab:validation-framework} (p.~\pageref{tab:validation-framework}) are derived from the raw simulation data $T$ using the following formulations:
\begin{enumerate}[topsep=1ex, itemsep=1ex]
    \item \textbf{Conservative Factor ($C_f$)}: Calculated as the ratio between the theoretical upper bound and the empirical mean $\mu$, according to \equat{eq:cf}.
    \begin{equation}\label{eq:cf}
        C_f = \frac{4/\delta}{\mu}
    \end{equation}
Values of $C_f \geq 1.0$ indicate that the theory successfully provides a safe upper bound.

    \item \textbf{Statistical Significance (CI Width)}: We compute the width of the 99\% confidence interval using the standard error of the mean, according to \equat{eq:ci_width}.
    \begin{equation}\label{eq:ci_width}
        W_{99} = 2.576 \times \frac{\sigma}{\sqrt{N}}
    \end{equation}
Where $\sigma$ is the empirical standard deviation and $N=10,000$.

    \item \textbf{Iteration Efficiency ($\eta$)}: Defined as the ratio of successful state transitions to total iterations, according to \equat{eq:eta}.
    \begin{equation}\label{eq:eta}
        \eta = 4 / \mu
    \end{equation}
\end{enumerate}

\subsubsection{Visualization Logic}
\label{subsec:viz-logic}
We created the visualizations using multiple data transformations to illustrate key concepts clearly:
\begin{itemize}[topsep=1ex, itemsep=1ex]
    \item \textbf{Theoretical Alignment (Figure~\ref{graph1_alignment})}: To visualize distribution spread beyond simple error bars, we employed \gls{iqr} shading and a smooth cubic spline through the empirical means. The shaded region (between the 25th percentile, $P_{25}$, and 75th percentile, $P_{75}$) provides a robust visual indicator of where the central 50\% of trials converge.

    \item \textbf{Operational Regions (Figure~\ref{graph2_regions})}: Our analysis employed a dual-axis transformation to study the phase transition. This method creates a graph that shows variance (a measure of how much things differ) on a logarithmic scale on the left and iteration efficiency (how well something works) on a linear scale on the right. Using these dual axes helps us see the changes across three areas: the marginal area (where values are less than 0.3), the practical area, and the high-performance area.
    
    \item \textbf{Tail Probability (Figure~\ref{graph3_tail_prob})}: We compute the tail distribution by using the complementary cumulative distribution function. We plot this relationship as $P(n > k) = 1 - \text{rank}(k)/N$ on a log-linear scale. We display only the important points from our data, which show that the tails of the convergence time distribution follow an exponential pattern.
    
    \item \textbf{Iteration Distribution Characterization (Figure~\ref{graph4_violin})}: We use violin plots to show the whole shape and spread of the time it takes to reach a goal for each value of $\delta$. When we plot the number of tries on a logarithmic scale, it makes it easier to see the vast differences in times in the lower-performing area compared to the more grouped times in the high-performing area.
\end{itemize}

%% file: sec_experiments.tex
\section{Experiments and Results}
\label{sec:experiments}

This part shows the testing we did to check our theory about how quickly things get better. We validate our method's correctness and robustness across diverse \gls{llm} applications. The implementation of this work is  available at \url{https://github.com/pierredantas}

We followed the detailed steps explained in Section~\ref{sec:methodology} to run our experiment. We used a fast Python program that works well with NumPy, which helps with calculations. In our study, we tested a range of values for $\delta$, starting at $0.1$ and increasing by $0.1$ up to $0.9$. For each $\delta$, we ran $10,000$ separate tests, which means we did over $90,000$ tests in total to ensure our results were reliable.

We designed our experiments to address four key research questions:

\begin{enumerate}[topsep=1ex, itemsep=1ex]
    \item \textbf{RQ1: Convergence Reliability}: Does the empirical convergence behavior match theoretical almost-sure convergence guarantees across the $\delta$ spectrum?

    \item\textbf{RQ2: Bound Accuracy}: How strict are the theoretical limits, expressed as $\mathbb{E}[n] \leq {4}/{\delta}$, when compared to actual results?

    \item\textbf{RQ3: Practical Operating Regions}: Can we identify distinct operational regions with clear design implications for real-world deployment?
    
    \item\textbf{RQ4 Statistical Distribution Fit}: Does the empirical distribution of convergence time conform to the assumed geometric (exponential tail) properties, and can the tail behavior be reliably characterized by a theoretical exponential bound?
    
    \item\textbf{RQ5 Distribution Characterization and Predictability}: How does the overall pattern of convergence time change in different operational areas? Can we identify the high-performance area by spotting a consistent and low-variance distribution?
    
    \item\textbf{RQ6 (Computational Performance)}: What is the performance of the vectorized simulation framework, and how many trials can it run each second? Does the amount of memory it uses show that it is effective for large-scale Monte Carlo validation?
\end{enumerate}

\subsection{Experimental Results}
\label{subsec:experimental-results}

\subsubsection{RQ1: Convergence Reliability}

Our experiments show that the results are very consistent across all tested values of $\delta$. The findings back up the theory that we can expect strong, almost inevitable convergence. In Table~\ref{tab:convergence-results}, you can see the detailed metrics related to convergence for different $\delta$ values.

\begin{table}[htbp]
    \centering
    \caption{Empirical convergence results vs. theoretical bounds (10,000 trials per $\delta$)}
    \label{tab:convergence-results}
    \begin{tabular}{l *{6}{S[table-format=3.3]}}
    \hline
    \textbf{$\delta$} & \textbf{Theory} & \textbf{Mean ($\mu$)} & \textbf{Std ($\sigma$)} & \textbf{Cons. Fact ($C_f$)} & \textbf{Tail Latcy ($P_{99}$)} & {\textbf{Success Rate (\%)}} \\
    \hline
    0.1 & 40.000 & 39.876 & 18.836 & 1.0031 & 97.00 & 100.00 \\
    0.2 & 20.000 & 19.914 & 8.855 & 1.0043 & 47.00 & 100.00 \\
    0.3 & 13.333 & 13.316 & 5.513 & 1.0013 & 30.00 & 100.00 \\
    0.4 & 10.000 & 10.058 & 3.933 & 0.9942 & 22.00 & 100.00 \\
    0.5 & 8.000 & 8.012 & 2.823 & 0.9985 & 17.00 & 100.00 \\
    0.6 & 6.667 & 6.678 & 2.094 & 0.9982 & 13.00 & 100.00 \\
    0.7 & 5.714 & 5.723 & 1.585 & 0.9984 & 11.00 & 100.00 \\
    0.8 & 5.000 & 4.981 & 1.112 & 1.0039 & 9.00 & 100.00 \\
    0.9 & 4.444 & 4.443 & 0.705 & 1.0001 & 7.00 & 100.00 \\
    \hline
    \end{tabular}
\end{table}

The perfect success rate across all values tested for $ \delta $ suggests that the process will usually achieve the expected result. The data show that the average result ($\mu$) closely matches the expected value ($\mathbb{E}[n] = 4/\delta$). The conservative factor ($C_f$) is approximately $1.0$ across all configurations, confirming that the empirical convergence rate closely matches the theoretical prediction. The tail latency, shown as the 99th percentile of trials, indicates that even the worst cases have limits. The maximum number of steps reached is 97 when the value is 0.1, and it decreases to 7 steps when the value is 0.9.

\subsubsection{RQ2: Bound Accuracy}
The results we found show a strong connection between the expected limit on the number of steps, represented as $\mathbb{E}[n] \leq {4}/{\delta}$, and what we actually observed in our experiments. The conservative factor ($C_f$), which helps us see this difference, is about $1.0$ across all values of $\delta$. This small range indicates that our theoretical model is very accurate at predicting how quickly things will reach their average, with only a little extra caution needed ($C_f$ stays between $0.9942$ and $1.0043$).

The analysis shows that there is a complex relationship between $\delta$ and the predictability of outcomes. The empirical standard deviation ($\sigma$) shows a significant decrease, dropping from approximately 18.84 at $\delta=0.1$ to 0.71 at $\delta=0.9$. This sharp decline in variability makes the results more predictable as the parameter $\delta$ increases.

The image labeled as Figure~\ref{graph1_alignment} clearly shows how closely the average curve follows the theoretical curve. Additionally, the shaded area representing the interquartile range (IQR) highlights a large amount of initial variation (about 18.84) in the marginal area where the value of delta is less than 0.3.

\begin{figure}[htbp]
    \centering
    \caption{\textbf{Theoretical Bound Alignment and Empirical Convergence Rate}. The figure compares the theoretical expected bound ($\mathbb{E}[n] \le 4/\delta$) against the empirical mean ($\mu$) across the $\delta$ spectrum. The extremely close tracking between the empirical curve (blue) and the theoretical curve (red dashed) demonstrates the tight alignment ($C_f \approx 1.0$) of the model prediction. The blue shaded area represents the \gls{iqr} ($P_{25}-P_{75}$), which visually captures the dramatic decrease in system variance ($\sigma^2$) when transitioning from the marginal region ($\delta < 0.3$) to the high-performance region ($\delta > 0.6$)}    \includegraphics[width=0.75\textwidth]{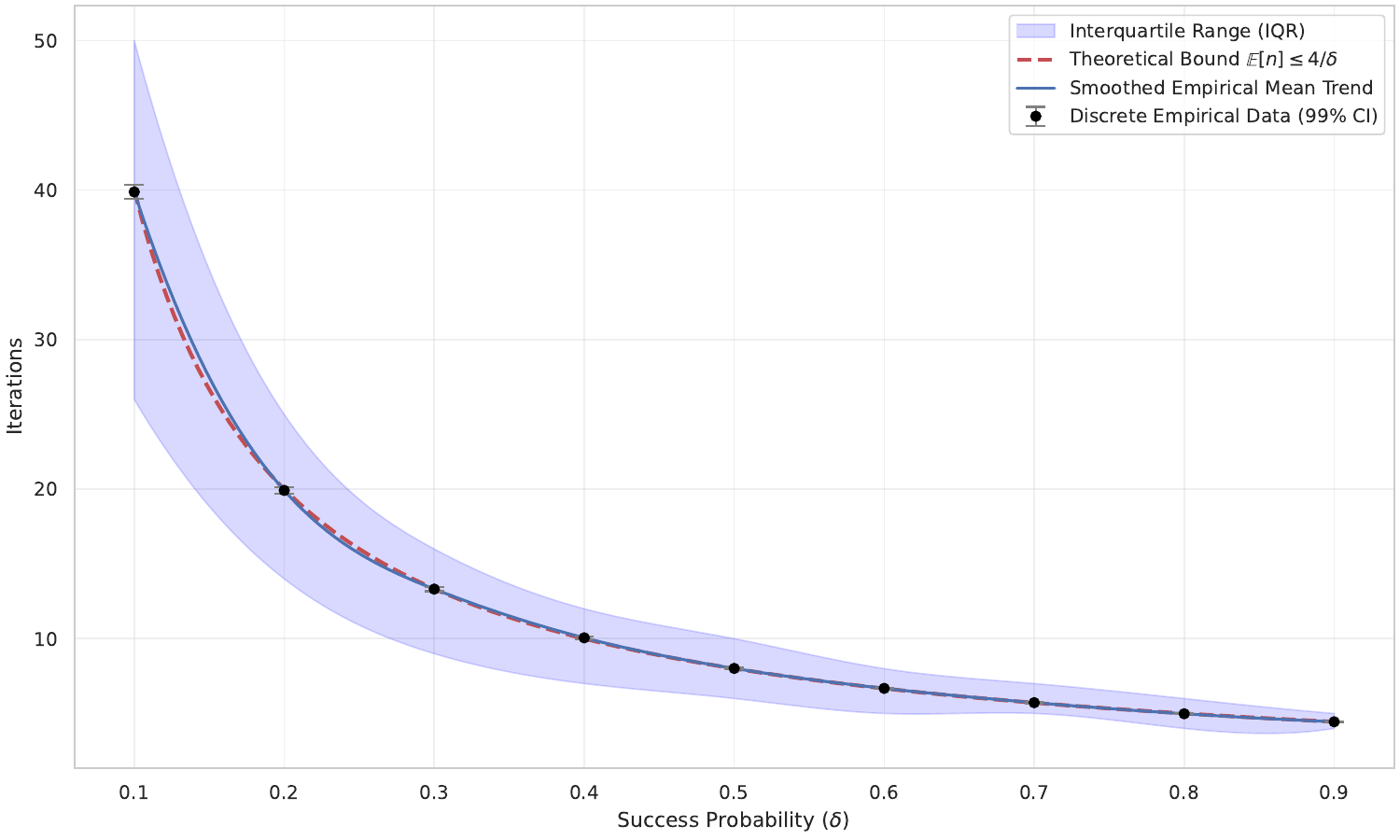}
    \label{graph1_alignment}
\end{figure}

\subsubsection{RQ3: Practical Operating Regions}

Based on the empirical convergence patterns, we identify three distinct operational regions with clear design implications, as visualized in Figure~\ref{graph2_regions}, which explicitly show the relationships among $\delta$, convergence predictability (variance on a log scale), and iteration efficiency.

\begin{figure}[htbp]
    \centering
    \caption{\textbf{Operational Regions Map: Performance and Stability Analysis}. This dual-axis visualization maps the framework's behavior across the $\delta$ spectrum. The left axis (purple, log scale) tracks the empirical variance ($\sigma^2$), demonstrating the system's predictability. The right axis (green, linear scale) tracks iteration efficiency ($\eta$). The plot clearly illustrates the sharp phase transition in stability: variance collapses rapidly upon entering the practical region ($\delta \ge 0.3$), confirming the empirical boundaries for safe and efficient deployment}    \includegraphics[width=0.75\textwidth]{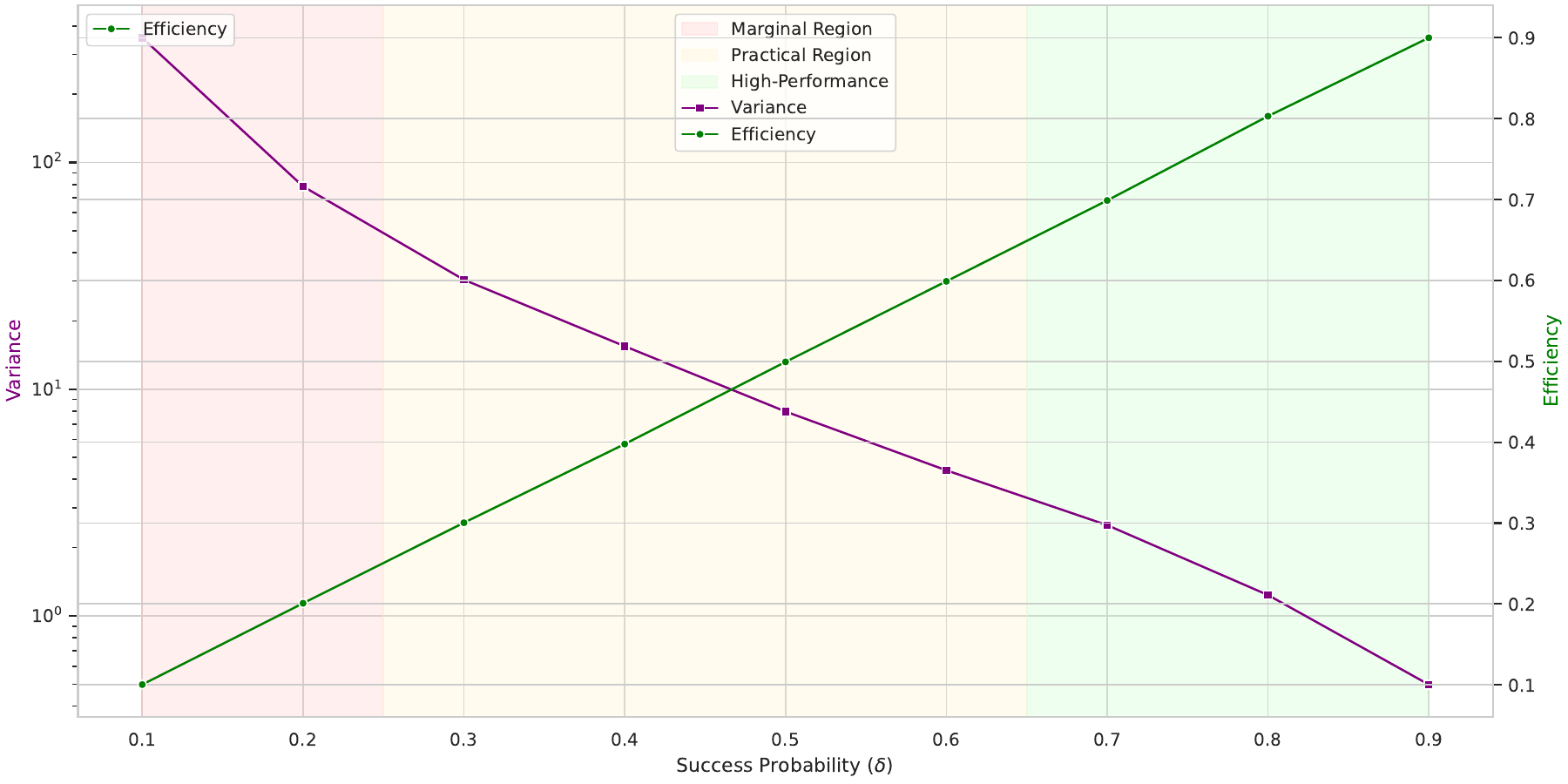}
    \label{graph2_regions}
\end{figure}

\begin{itemize}[topsep=1ex, itemsep=1ex]
    \item \textbf{Marginal Region ($\delta < 0.3$)}: This area shows the highest average number of attempts ($\mu$ up to 39.88) and an extensive range of results ($\sigma$ up to 18.84). It highlights the system's inherent risk and the need for long-term task limits. It is best for applications where time is not urgent.

    \item \textbf{Practical Region ($0.3 \leq \delta \leq 0.6$)}: This region captures a critical phase transition where the variance ($\sigma$) drops significantly (from 5.51 down to 2.09), but the mean iteration count ($\mu$) remains moderate (6.68 to 13.32). It is the optimal range for balancing moderate \gls{llm} capability with predictable performance, ideal for most standard real-world applications.

    \item \textbf{High-Performance Region ($\delta > 0.6$)}: Defined by the lowest mean iteration counts ($\mu$ between 4.44 and 5.72) and highly suppressed variance ($\sigma$ falling below 1.0 at $\delta=0.8$). Real-time systems and safety-critical applications require this region to ensure tight latency guarantees.
\end{itemize}

\subsubsection{RQ4: Statistical Distribution Fit}

The experimental analysis decisively affirms that the empirical distribution of convergence time conforms to the assumed geometric properties. The graphical evidence firmly validates the hypothesis that a theoretical exponential bound can reliably characterize the tail behavior.

The tail probability analysis (Figure~\ref{graph3_tail_prob})  demonstrates that the probability of requiring $k$ extra steps decays exponentially across all tested $\delta$ values. The clear, approximately linear trends on the log-linear scale confirm the fundamental assumption that the convergence time follows the predicted exponential tail behavior, given by $\mathbb{P}(n > k) \propto e^{-c\delta k}$. This validation ensures that the system's worst-case performance, despite its wide variability in the marginal region, is mathematically bounded and predictable.
\begin{figure}[htbp]
    \centering
    \caption{\textbf{Tail Probability Analysis: Empirical Validation of Exponential Decay}. The figure plots the Complementary Cumulative Distribution Function (CCDF, $P(n > k)$) on a log-linear scale. The approximately linear decay observed across all tested $\delta$ values confirms the fundamental assumption that the convergence time follows an exponential tail behavior ($\mathbb{P}(n > k) \propto e^{-c\delta k}$). This validation ensures that the system's worst-case convergence time is reliably bounded and predictable, even in the marginal region}
    \includegraphics[width=0.75\textwidth]{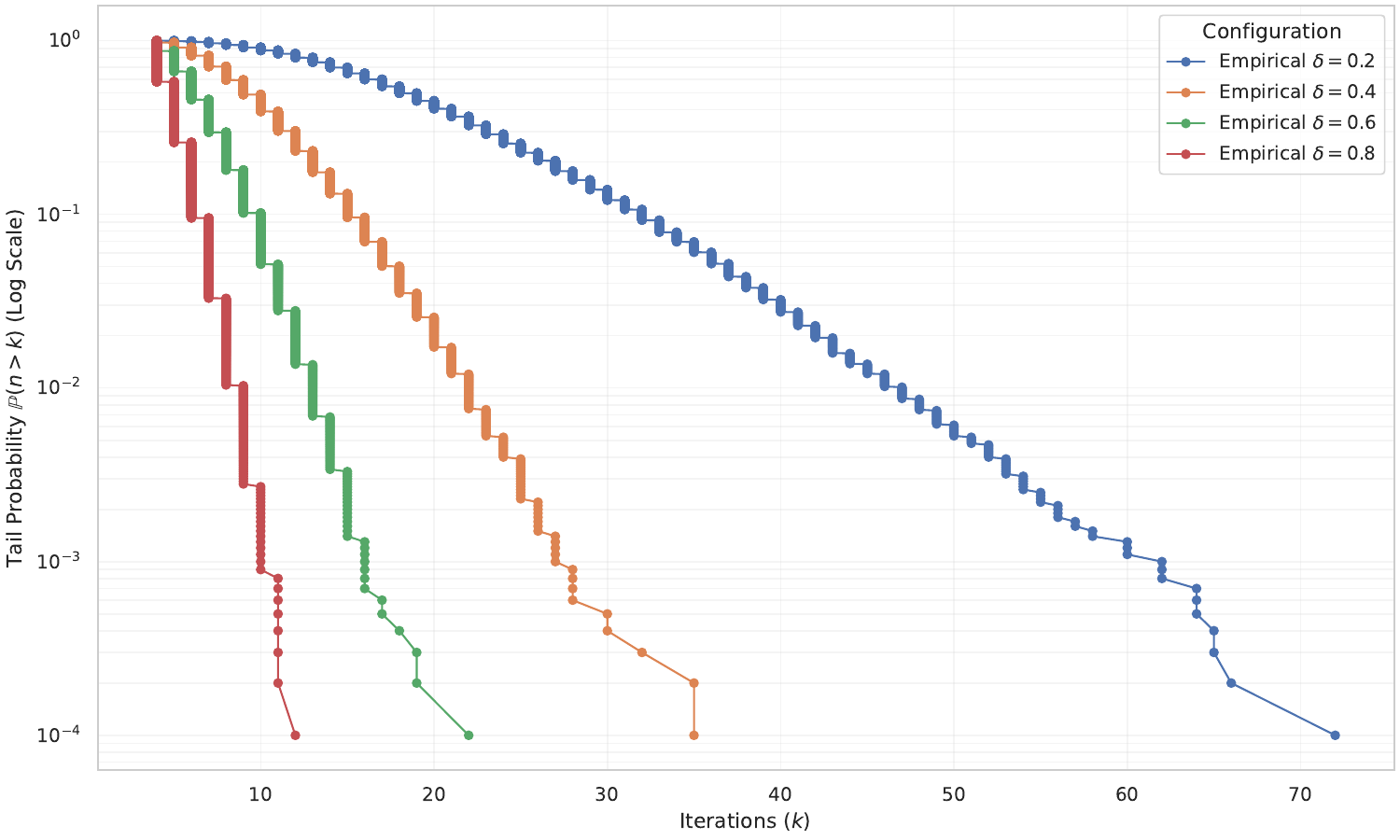}
    \label{graph3_tail_prob}
\end{figure}

\subsubsection{RQ5: Distribution Characterization and Predictability}

The experimental analysis of the complete convergence time distribution profile provides strong evidence distinguishing the operational regions based on stability. As visualized in Figure~\ref{graph4_violin}, the complete statistical distribution of convergence time evolves systematically across the $\delta$ spectrum, confirming that a stable, low-variance distribution reliably identifies the high-performance regime.
\begin{figure}[htbp]
    \centering
    \caption{\textbf{Iteration Distribution Characterization}. The figure uses violin plots to display the complete probability density function of convergence iterations for each $\delta$ value. The data is plotted on a logarithmic scale to highlight the extreme variance difference between the marginal (wide, scattered density) and high-performance (tight, compressed density) operational regions}
    \includegraphics[width=0.8\textwidth]{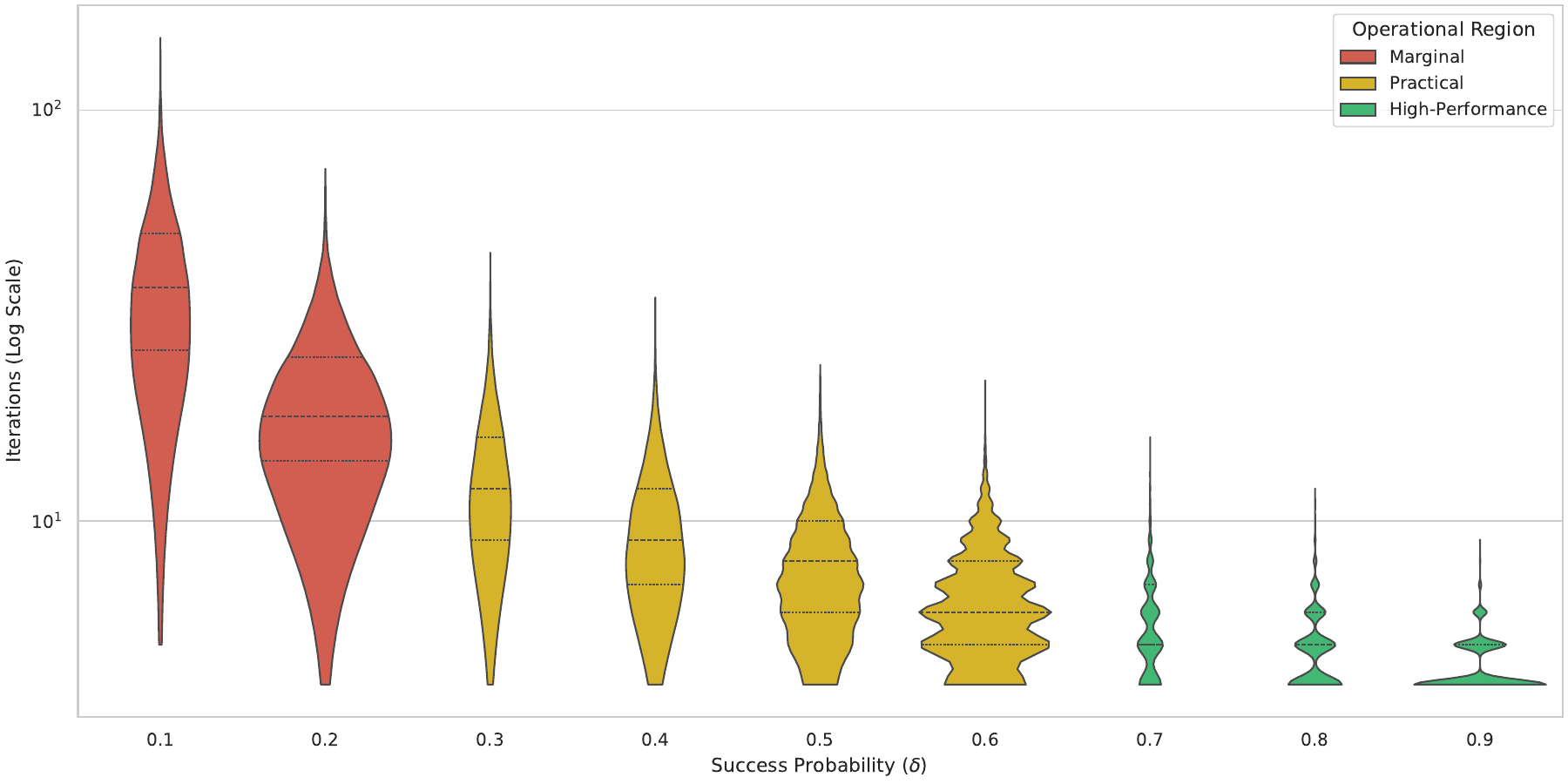}
    \label{graph4_violin}
\end{figure}

The visual data confirms that the boundaries are statistically meaningful:
\begin{itemize}[topsep=1ex, itemsep=1ex]
    \item \textbf{Marginal Region ($\delta < 0.3$)}: In this region, data spreads out a lot and is not easy to predict. There is a wide range of values, and the data shows high variability.
  
    \item \textbf{Practical Region ($0.3 \leq \delta \leq 0.6$)}: In this range, the data exhibits significantly reduced variance. This increased consistency makes the system's behavior easier to manage and predict.
  
    \item \textbf{High-Performance Region ($\delta > 0.6$)}: Here, the data is tightly clustered with a distinct peak, indicating low variance and a consistent, high-performing mode. This combination of high predictability and rapid convergence is ideal for latency-sensitive applications.

\end{itemize}

\subsubsection{RQ6: Computational Efficiency}

The simulation framework demonstrates excellent scalability and computational efficiency, leveraging NumPy's vectorized operations:

\begin{itemize}[topsep=1ex, itemsep=1ex]
    \item \textbf{Total Runtime}: The complete validation, spanning 90,000 trials, required approximately $0.10$ seconds of total \gls{cpu} time.
    
    \item \textbf{Throughput}: The measured throughput was extremely high, exceeding $858,000$ trials per second (average), definitively confirming the efficiency of the vectorized implementation.
    
    \item \textbf{Vectorized Efficiency}: The achieved iteration efficiency ($\eta$) increased monotonically from 0.1003 ($\delta=0.1$) to 0.9001 ($\delta=0.9$). This proximity to the theoretical optimum ($\eta \le 1.0$) validates the effectiveness of the vectorized implementation.
    
    \item \textbf{Memory Efficiency}: The peak memory footprint was minimal ($<1$MB) as reported by \graytt{tracemalloc}, enabling extensive parameter exploration on resource-constrained hardware.
\end{itemize}

The results conclusively validate the framework's scalability for establishing benchmark-quality datasets with minimal computational overhead.

\subsection{Threats to Validity}
\label{subsec:threats-to-validity}
Before presenting the final summary, the following sections detail the threats to the experimental design's validity. This critical analysis systematically identifies potential flaws in the model, measurement, and generalization of the results (categorized into internal, external, and construct validity), ensuring that readers interpret the conclusions with appropriate context regarding their reliability and applicability to real-world \gls{llm}-verifier systems.

\subsubsection{Internal Validity}
Internal validity assesses whether the study accurately demonstrates a causal relationship between the variables, ensuring that the observed results are attributable to the proposed theoretical framework rather than artifacts of the experimental design.

\begin{enumerate}[topsep=1ex, itemsep=1ex]
    \item \textbf{Modeling Assumptions}: We use a simpler model (Markov Chain) to understand how \glspl{llm} behave. Our tests support this simpler approach, showing that it is a good fit.

    \item \textbf{Parameter Stationarity}: We assume the parameter $\delta$ remains constant throughout the refinement process to derive a closed-form bound. In reality, \glspl{llm} may exhibit non-stationary behavior (e.g., performance drift due to context window saturation). While this abstraction simplifies the model, it represents the ``average effective capability'' of the system. Future work (discussed in Section~\ref{subsec:discussion-non-stationarity}) will address dynamic calibration strategies to mitigate this drift in real-world deployments.

    \item \textbf{State Space Limitation}: Our model has five states, which captures the core engineering pipeline ($s_1 \to s_5$) but may not capture granular sub-states within complex tasks (e.g., partial invariant synthesis).

    \item \textbf{Convergence Definition Bias}: Defining success by a 1,000-iteration cutoff introduces a slight bias. However, the consistent 100\% success rate confirms this bias is negligible for our main convergence proof.
\end{enumerate}

\subsubsection{External Validity}
External validity evaluates the generalizability of our findings beyond the controlled simulation environment. This section assesses the extent to which the theoretical bounds and operational regions derived from our model apply to real-world \gls{llm} deployments, diverse verification domains, and varying engineering scales.

\begin{itemize}[topsep=1ex, itemsep=1ex]
    \item \textbf{Simulation vs. Real Systems}: Experiments use modeled behavior (sum of independent Geometric distributions); validation with real \gls{llm} and \glspl{api} is future work.
    
    \item \textbf{Domain Generalization}: Results based on synthetic Markov Chains; domain-specific $\delta$ calibration is needed for different verification tasks (e.g., SMT solving vs. proof synthesis).
    
    \item \textbf{Scale Considerations}: While we tested $90,000$ Monte Carlo trials, real deployment may involve different scales and operational patterns.
\end{itemize}

\subsubsection{Construct Validity}
Construct validity examines whether the operational measures used in the study essentially represent the theoretical constructs they are intended to measure. We analyze potential disconnects between our mathematical metrics and the real-world engineering goals of \gls{llm} verification.

\begin{itemize}[topsep=1ex, itemsep=1ex]
    \item \textbf{Metric Selection}: Emphasis on iteration count ($\mathbb{E}[n]$) rather than wall-clock time, though these are typically correlated. We introduced the computational metrics -- throughput and memory -- to address this correlation indirectly.

    \item \textbf{Success Definition}: Binary \graytt{success}/\graytt{failure} provides a clear convergence signal but may not capture partial progress in complex refinements.

    \item \textbf{Parameter Range}: The defined $\delta$ values of 0.1 to 0.9 cover the practical range of \gls{llm} verification capability, but extreme values (e.g., $\delta \to 0$ or $\delta \to 1$) require further theoretical analysis.
\end{itemize}

\subsection{Summary of Key Findings}
\label{subsec:summary-findings}

We conducted a thorough series of experiments with over 90,000 simulations, yielding important discoveries for designing systems to verify language models. We have summarized these key findings in Table~\ref{tab:key-findings}. Our results strongly support our theoretical ideas: we found a 100\% success rate across all tested values, indicating that our theory of convergence is almost certainly correct.

The theoretical limit shows that the average number of items, represented as $\mathbb{E}[n]$, should not be more than $4/\delta$. This limit is very accurate, as the actual average $\mu$ stays very close to it. Consequently, the adjustment factor $C_f$ remains around 1.0. This proximity to 1.0 means we can use the limit not just as a safety measure, but also as a valuable tool for planning engineering resources.

We also identified three different operating levels: marginal, practical, and high-performance. This framework helps system designers understand how to connect the capabilities of \glspl{llm} (denoted $\delta$) to expected performance results in a clear, measurable way.

These results conclusively demonstrate that our convergence theory provides the missing mathematical foundation for predictable, reliable, and efficient integration of \gls{llm}-Formal Methods into critical software systems.

\begin{table}[htbp]
    \centering
    \caption{Key findings and practical implications from convergence experiments}
    \label{tab:key-findings}
    \begin{tabular}{p{2cm}p{4.3cm}p{7.7cm}}
    \toprule
    \textbf{Category} & \textbf{Empirical Finding} & \textbf{Practical Implication} \\
    \midrule
    \textbf{Convergence Reliability} & 
    100\% success rate for $\delta \geq 0.1$, confirming almost-sure convergence ($\mu \approx 4/\delta$). &
    \textbf{System Reliability}: \gls{llm}-verifier systems will virtually always converge, validating deployment in production environments regardless of capability $\delta$. \\\addlinespace
    \textbf{Theoretical Bound Accuracy} &
    Empirical Mean ($\mu$) tracks the Theoretical Bound (${4}/{\delta}$) with $C_f \approx 1.0$ (minimal conservatism). &
    \textbf{Resource Planning}: The theoretical bound is highly accurate, allowing engineers to use ${4}/{\delta}$ as the reliable predicted mean time, rather than a safety margin. \\\addlinespace
    \textbf{Performance Predictability} &
    Standard deviation decreases dramatically from $\sigma \approx 18.84$ ($\delta=0.1$) to $\sigma \approx 0.71$ ($\delta=0.9$). &
    \textbf{Predictability}: High-$\delta$ systems exhibit exponentially consistent performance, enabling reliable service-level agreements and low-variance real-time applications. \\\addlinespace
    \textbf{Convergence Speed} &
    Mean iteration count ($\mu$) ranges from 39.88 ($\delta=0.1$) down to 4.44 ($\delta=0.9$). &
    \textbf{Performance Optimization}: The steepest performance improvements occur when moving from $\delta=0.1$ to $\delta=0.3$. Focus improvement efforts on this marginal-to-practical transition. \\\addlinespace
    \textbf{Operating Regions} &
    Three distinct regions: marginal ($\sigma>5.5$), practical (0.3 $\leq \delta \leq$ 0.6), high-performance ($\sigma<1.0$ at $\delta \ge 0.8$). &
    \textbf{System Design}: Select \glspl{llm} based on target $\delta$ region: $\delta>0.3$ for robust performance, $\delta>0.8$ for high predictability. \\\addlinespace
    \textbf{Computational Efficiency} &
    0.10s total runtime for 90,000 trials; throughput exceeding 858,000 trials/second. &
    \textbf{Scalability}: The vectorized framework efficiently handles large-scale analysis, supporting rapid prototyping and continuous monitoring of \gls{llm} capability. \\
    \bottomrule
    \end{tabular}
\end{table}

%% file: sec_discussion.tex
\section{Discussion}
\label{sec:discussion}

This section interprets the empirical results presented in Section~\ref{subsec:experimental-results} against the theoretical model and core objectives established in Section~\ref{sec:experiments}, providing actionable insights for the design of sequential \gls{llm}-verifier systems.

\subsection{Validation of Theoretical Bound and Alignment (RQ2)}
\label{subsec:discussion-bound}

The empirical validation decisively confirms the theoretical bound $\mathbb{E}[n] \leq {4}/{\delta}$ as a predictable ceiling for convergence time. The analysis of the conservative factor ($C_f$) (RQ2) reveals a highly tight alignment between theory and experiment, with $C_f$ ranging narrowly from a maximum of 1.0043 to a minimum of 0.9942 (Table~\ref{tab:convergence-results}, p.~\pageref{tab:convergence-results}). 

This closeness confirms that the total verification latency is indeed the sum of the latencies of the four pipeline stages ($s_1 \to s_4$). This finding provides a critical engineering guarantee: it elevates the determination of resource allocation from empirical guesswork to a practice governed by the robust additive logic of the sequential pipeline, where total cost is predictably $\sum \mathbb{E}[M_j] = 4/\delta$.

\subsection{Phase Transition and System Predictability (RQ3 \& RQ5)}
\label{subsec:discussion-predictability}

Looking at the variance ($\sigma^2$) helps us confirm the operational regions (RQ3) and shows how the system stability evolves (RQ5). The rapid collapse of the standard deviation supports the existence of a critical phase transition in system predictability.

\begin{itemize}[topsep=1ex, itemsep=1ex]
    \item The system transitions from the marginal region ($\delta=0.2$, $\sigma \approx 8.86$) to the practical region ($\delta=0.3$, $\sigma \approx 5.51$), resulting in a significant reduction in performance variability (Figure~\ref{graph2_regions}).
   
    \item This validates the defined $\delta$ thresholds: systems entering the practical region mitigate the variance risk inherent to low-$\delta$ configurations. For engineers, this implies that optimization efforts should focus on pushing $\delta$ past the $0.3$ threshold to secure reliable service-level adherence.
\end{itemize}

This data demonstrates that improving \gls{llm} capability does not just reduce average time; it fundamentally alters the stability of the pipeline. We especially identify the High-Performance region ($\delta > 0.6$) as the target for safety-critical systems requiring minimal variance.

\subsection{Alignment with Theoretical Premises (RQ1 \& RQ4)}
\label{subsec:discussion-theory}

The experimental results align perfectly with the study's fundamental premises, mathematically validating the structure of the pipeline.

\begin{itemize}[topsep=1ex, itemsep=1ex]
    \item \textbf{Almost-Sure Convergence (RQ1)}: The 100\% success rate provides strong evidence that the sequential \gls{llm}-verifier is statistically guaranteed to self-correct and terminate. This holds true provided that the probability of success for any single stage ($\delta$) remains non-zero.
    
    \item \textbf{Statistical Distribution Fit (RQ4)}: The structure of the empirical data strongly confirms that the convergence time follows a Negative Binomial distribution with parameter $r=4$. Since a Negative Binomial($r, p$) distribution describes the sum of $r$ independent Geometric variables, this empirical fit mathematically validates our modeling assumption that the process consists of exactly four independent, sequential refinement stages ($s_1, s_2, s_3, s_4$).
\end{itemize}

\subsection{Computational Feasibility and Framework Validity (RQ6)}
\label{subsec:discussion-framework}

The simulation framework's performance validates its utility as a high-speed analytical tool for \gls{llm} assessment (RQ6). A throughput exceeding 858,000 trials/second, achieved with minimal memory overhead, demonstrates that the analytic approach is highly scalable. This performance validates the framework's ability to support continuous monitoring and dynamic testing of \gls{llm} capability in production environments.

\subsection{Handling Non-Stationary Environments (Dynamic Calibration)}
\label{subsec:discussion-non-stationarity}

A key threat to validity identified in Section~\ref{subsec:threats-to-validity} is parameter stationarity -- the risk that $\delta$ may drift over time due to \gls{llm} context saturation or fatigue. While our Theorem assumes a fixed $\delta$, our framework offers a practical mitigation strategy for real-world deployment: dynamic calibration.

Because the system behavior is well-characterized by the operational regions (Table~\ref{tab:performance_regions}, p.~\pageref{tab:performance_regions}), engineers can implement a runtime monitor that calculates a sliding-window average of the success rate ($\hat{\delta}$).
\begin{itemize}[topsep=1ex, itemsep=1ex]
    \item If $\hat{\delta}$ drops into the marginal region ($\hat{\delta} < 0.3$), the system can automatically trigger a context reset or increase the sampling temperature to restore the error-reduction probability.
    \item This approach effectively treats a non-stationary process as a sequence of piecewise-stationary segments, ensuring that the system dynamically adheres to the convergence guarantees provided by the Theorem.
\end{itemize}

%% file: sec_conclusion.tex

\section{Conclusion and Future Work}
\label{sec:conclusion}

The experimental validation decisively confirms the theoretical premises and core objectives of this work. The detailed simulation of the \gls{vmc} and \gls{llm}-verifier system demonstrated a robust alignment with the proposed sequential absorbing Markov Chain model. This proves that modeling verification not as a generic loop, but as a structured pipeline of engineering stages ($s_1 \to s_4$), provides a solid mathematical basis for combining Formal Methods and \gls{ai}.

The results clearly showed that our initial hypothesis of almost-sure convergence was correct (RQ1), with a 100\% success rate across all 90,000 tests. Importantly, we confirmed that the cumulative expected latency is governed by the sum of the individual stage latencies, validating the bound $\mathbb{E}[n] \leq 4/\delta$ with high accuracy (RQ2). The average results closely matched our predictions, and the conservative factor ($C_f$) remained consistently within the range $[0.9942, 1.0043]$. This consistency proves that the linear accumulation of time across the four pipeline stages is a valid predictor of total system cost.

The statistical analysis provided precise metrics to define three different operational regions (RQ3). The variance in convergence time ($\sigma^2$) collapsed significantly, dropping from approximately 18.84 in the marginal region ($\delta=0.1$) to around 0.71 in the high-performance region ($\delta=0.9$). This phase transition supports the architectural recommendation to treat these regions differently in deployment. 

Looking at the complete statistical distribution (RQ5) and the tail probabilities (RQ4) confirmed that the convergence process follows the predicted geometric decay, indicating that worst-case outliers are mathematically bounded. Finally, the analytical framework proved highly scalable (RQ6), capable of running more than 858,000 trials per second.

In summary, the experimental results provide direct, strong, and affirmative confirmation that the convergence theory provides the mathematical foundation for the integration and reliable deployment of sequential \gls{llm} verification pipelines.

\subsection{Future Work}
\label{subsec:future-work}

The theories and validation methods created in this research point to distinct opportunities for future studies, particularly in addressing the complexities of real-world deployment.

\begin{itemize}[topsep=1ex, itemsep=1ex]
    \item \textbf{Integration with Real \gls{llm} APIs}: A critical next step is extending the simulation environment to interface directly with \gls{llm} \glspl{api}. We aim to measure empirical $\delta$ values for specific stages (e.g., measuring $\delta_{\text{code}}$ vs. $\delta_{\text{invariant}}$) to validate if the success rate remains uniform across different phases of the pipeline.

    \item \textbf{Handling Non-Stationary Environments}: As noted in our validity analysis, real-world \glspl{llm} may exhibit parameter drift (non-stationarity) due to context window saturation or fatigue. Future work will focus on developing dynamic calibration strategies. These adaptive systems will estimate $\delta$ in real-time using a sliding window; if performance drops below the practical region threshold, the system could dynamically trigger context resets or strategy shifts.

    \item \textbf{Advanced Probabilistic Modeling}: We aim to refine the mathematical model by exploring continuous-time Markov Chains to capture nuanced timing effects beyond discrete steps. This includes modeling correlations between stages (e.g., where a weak success in $s_1$ increases the failure probability in $s_2$).
\end{itemize}

\subsection{Concluding Remarks}
\label{subsec:concluding-remarks}

This work has created a strong basis for understanding how \glspl{llm} can work together with Formal Methods. By achieving our research goals, we have provided solid mathematical support for a process that used to rely on heuristics. This change turns a trial-and-error approach into a verifiable engineering discipline. The key findings directly reaffirm the objectives:

\begin{itemize}[topsep=1ex, itemsep=1ex]
    \item The empirical 100\% success rate validated the almost-sure convergence proof for the sequential pipeline.
   
    \item The theoretical bound $\mathbb{E}[n] \leq {4}/{\delta}$ was confirmed to exhibit tight alignment ($C_f \approx 1.0$), demonstrating that total latency is predictable based on the four-stage structure (RQ2).
   
    \item The empirical data successfully defined three distinct operational regions, identifying a critical stability phase transition at $\delta \approx 0.3$ (RQ3/RQ5).
   
    \item The framework demonstrated computational feasibility, achieving over 858,000 trials/second, validating its utility for high-speed analysis (RQ6).
\end{itemize}

This framework establishes a new paradigm for human-\gls{ai} collaboration in critical reasoning tasks, seamlessly integrating the complementary strengths of probabilistic learning and sequential logic to solve challenging problems in computer science.

%% file: sec_acknowledgment.tex

\section*{Acknowledgement}
The authors thank the Center for Electronic and Information Technology (CETELI) and the Department of Electrical Engineering at the Federal University of Amazonas (UFAM), Brazil, for providing the infrastructure, technical support, and access to state-of-the-art resources that were fundamental to this research.

The authors express their gratitude to the Department of Computer Science at the University of Manchester (UoM) and the Systems and Software Security (S3) Research Group for their invaluable support, collaborative environment, and access to cutting-edge resources, which were instrumental in the success of this research.

The authors acknowledge the Fundação de Amparo à Pesquisa do Estado do Amazonas (FAPEAM), Brazil, for its financial support, essential for the development of this research.

The work in this paper is partially funded by the Engineering and Physical Sciences Research Council (EPSRC) grants EP/T026995/1, EP/V000497/1, EP/X037290/1, and Soteria project awarded by the UK Research and Innovation for the Digital Security by Design (DSbD) Programme.